\def\year{2021}\relax
\documentclass[letterpaper]{article} % DO NOT CHANGE THIS
\usepackage{aaai21}  % DO NOT CHANGE THIS
\usepackage{times}  % DO NOT CHANGE THIS
\usepackage{helvet} % DO NOT CHANGE THIS
\usepackage{courier}  % DO NOT CHANGE THIS
\usepackage[hyphens]{url}  % DO NOT CHANGE THIS
\usepackage{graphicx} % DO NOT CHANGE THIS
\usepackage{natbib}  % DO NOT CHANGE THIS AND DO NOT ADD ANY OPTIONS TO IT
\usepackage{caption} % DO NOT CHANGE THIS AND DO NOT ADD ANY OPTIONS TO IT
\usepackage{amsmath}
\usepackage{amsthm}
\usepackage{amssymb}
\usepackage{amsfonts}
\usepackage{mathtools}
\usepackage[shortcuts]{extdash}
\usepackage{mathabx}
\usepackage{MnSymbol}
\usepackage{hyperref}
\usepackage{xcolor}

\title{Complexity and Algorithms for Exploiting Quantal Opponents\\ in Large Two-Player Games}
\author{
    David Milec\textsuperscript{\rm 1}, Jakub \v{C}ern\'{y}\textsuperscript{\rm 2}, Viliam Lis\'{y}\textsuperscript{\rm 1}, Bo An\textsuperscript{\rm 2}
}

\affiliations{
   \\
   \textsuperscript{\rm 1} AI Center, FEE, Czech Technical University in Prague, Czech Republic \\
    \textsuperscript{\rm 2} Nanyang Technological University, Singapore \\
    milecdav@fel.cvut.cz, cerny@disroot.org, viliam.lisy@fel.cvut.cz, boan@ntu.edu.sg
}
\DeclareMathOperator*{\argmax}{arg\,max}
\DeclareMathOperator*{\softmax}{soft\,max}

\usepackage{tikz}
\usepackage{booktabs}
\usetikzlibrary{calc}

\newcommand{\tikzmark}[1]{\tikz[overlay,remember picture] \node (#1) {};}
\newcommand{\DrawBox}[3][]{%
    \tikz[overlay,remember picture]{
    \draw[black,#1]
      ($(#2)+(-0.25em,2.0ex)$) rectangle
      ($(#3)+(0.95em,-0.75ex)$);}
}
\definecolor{rqrcolor}{RGB}{246,82,43}
\definecolor{ntublue}{rgb}{0.094, 0.1098, 0.384}

\def\appendixurl{https://arxiv.org/abs/2009.14521}

\usepackage{xpatch}
\begin{document}
\makeatletter
\xpatchcmd{\paragraph}{3.25ex \@plus1ex \@minus.2ex}{3pt plus 1pt minus 1pt}{\typeout{success!}}{\typeout{failure!}}
\makeatother

\newcommand{\Iset}{\mathcal{I}}
\newcommand{\pr}{\vartriangle}
\newcommand{\ps}{\triangledown}
\def\real{\mathbb{R}}

% Environemnet styles
\newtheorem{example}{Example}
\newtheorem{theorem}{Theorem}
\newtheorem{definition}{Definition}
\newtheorem{observation}[theorem]{Observation}
\newtheorem{conjecture}[theorem]{Conjecture}
\newtheorem{corollary}[theorem]{Corollary}
\newtheorem{lemma}{Lemma}
\newtheorem{proposition}[theorem]{Proposition}

\newenvironment{sproof}{%
  \renewcommand{\proofname}{Proof (Sketch)}\proof}{\endproof}

\maketitle

\begin{abstract}
Solution concepts of traditional game theory assume entirely rational players; therefore, their ability to exploit subrational opponents is limited. One type of subrationality that describes human behavior well is the quantal response. While there exist algorithms for computing solutions against quantal opponents, they either do not scale or may provide strategies that are even worse than the entirely-rational Nash strategies. This paper aims to analyze and propose scalable algorithms for computing effective and robust strategies against a quantal opponent in normal-form and extensive-form games. Our contributions are: 
(1) we define two different solution concepts related to exploiting quantal opponents and analyze their properties; 
(2) we prove that computing these solutions is computationally hard;
(3) therefore, we evaluate several heuristic approximations based on scalable counterfactual regret minimization (CFR);
and (4) we identify a CFR variant that exploits the bounded opponents better than the previously used variants while being less exploitable by the worst-case perfectly-rational opponent.
%Real-world applications of game theory highlight the importance of modeling the bounded rationality of human players. 
%Previous research in normal-form and extensive-form games focused mainly on the case where the rationality of all players is bounded; however, the advances of artificial intelligence motivate studying situations where a perfectly rational player %faces a bounded rational opponent.
%We assume the opponent behaves based on a quantal response model and aim to propose scalable algorithms computing optimal strategies against such opponent in normal-form and extensive-form games.
%We prove that the problem is computationally hard and therefore study heuristic approximations of the desired solution based on counterfactual regret minimization (CFR).
%We study the relationship between the optimal and heuristic solutions both theoretically and empirically.
%Finally, we propose a novel variant of CFR, which exploits the bounded rational opponents better than the existing variants while being less exploitable by the worst-case rational opponent.
\end{abstract}

\section{Introduction}
Extensive-form games are a powerful model able to describe recreational games, such as poker, as well as real-world situations from physical or network security.
Recent advances in solving these games, and particularly the Counterfactual Regret Minimization (CFR) framework \cite{zinkevich2008regret}, allowed creating superhuman agents even in huge games, such as no-limit Texas hold'em with approximately $10^{160}$ different decision points \cite{moravvcik2017deepstack,brown2018superhuman}.
The algorithms generally approximate a Nash equilibrium, which assumes that all players are perfectly rational, and is known to be inefficient in exploiting weaker opponents. 
An algorithm that would be able to take an opponent's imperfection into account is expected to win by a much larger margin \cite{johanson2009data,bard2013online}.

The most common model of bounded rationality in humans is the quantal response (QR) model \cite{mckelvey1995quantal,mckelvey1998quantal}. Multiple experiments identified it as a \textit{good predictor of human behavior} in games~\cite{yang2012computing,haile2008empirical}. 
QR is also the \textit{hearth of the algorithms successfully deployed in the real world}~\cite{yang2012computing, fang2017paws}. 
It suggests that players respond stochastically, picking better actions with higher probability.
Therefore, we investigate \textbf{how to scalably compute a good strategy against a quantal response opponent in two-player normal-form and extensive-form games.}

If both players choose their actions based on the QR model, their behavior is described by quantal response equilibrium (QRE). Finding QRE is a computationally tractable problem  \cite{mckelvey1995quantal,turocy2005dynamic}, which can be also solved using the CFR framework \cite{farina2019online}. However, when creating AI agents competing with humans, we want to assume that \textbf{one of the players is perfectly rational, and only the opponent's rationality is bounded}.
A tempting approach may be using the algorithms for computing QRE and increasing one player's rationality or using generic algorithms for exploiting opponents \cite{davis2014using} even though the QR model does not satisfy their assumptions, as in \cite{basak2018initial}.
However, this approach generally leads to a solution concept we call Quantal Nash Equilibrium (QNE), which we show is very inefficient in exploiting QR opponents and may even perform worse than an arbitrary Nash equilibrium.

Since the very nature of the quantal response model assumes that the sub-rational agent responds to a strategy played by its opponent, a more natural setting for studying the optimal strategies against QR opponents are Stackelberg games, in which one player commits to a strategy that is then learned and responded to by the opponent.
Optimal commitments against quantal response opponents - Quantal Stackelberg Equilibrium (QSE) - have been studied in security games \cite{yang2012computing}, and the results were recently extended to normal-form games \cite{cerny2020}. Even in these one-shot games, polynomial algorithms are available only for their very limited subclasses. In extensive-form games, we show that computing the QSE is NP-hard, even in zero-sum games. Therefore, it is very unlikely that the CFR framework could be adapted to closely approximate these strategies. Since we aim for high scalability, we focus on empirical evaluation of several heuristics, including using QNE as an approximation of QSE. We identify a method that is not only more exploitative than QNE, but also more robust when the opponent is rational.

Our contributions are: \textbf{1)} We analyze the relationship and properties of two solution concepts with quantal opponents that naturally arise from Nash equilibrium (QNE) and Stackelberg equilibrium (QSE).
\textbf{2)}  We prove that computing QNE is PPAD-hard even in NFGs, and computing QSE in EFGs is NP-hard.
Therefore, \textbf{3)} we investigate the performance of CFR-based heuristics against QR opponents. The extensive empirical evaluation on four different classes of games with up to $10^8$ histories identifies a variant of CFR-$f$ \cite{davis2014using} that computes strategies better than both QNE and NE.

\section{Background}
Even though our main focus is on extensive-form games, we study the concepts in normal-form games, which can be seen as their conceptually simpler special case.
After defining the models, we proceed to define quantal response and the metrics for evaluating a deployed strategy's quality.

\subsection{Two-player normal-form games}
A two-player normal-form game (NFG) is a tuple $ G = (N,A,u)$ where $N = \{\pr, \ps \}$ is set of players. We use $i$ and $-i$ for one player and her opponent. $A = \{A_\pr,A_\ps\}$ denotes the set of ordered sets of \textit{actions} for both players. The \textit{utility function} $u_i:A_\pr \times A_\ps \to \mathbb{R}$ assigns a value for each pair of actions. A game is called zero-sum if $u_\pr=-u_\ps$.

\textit{Mixed strategy} $\sigma_i \in \Sigma_i$ is a probability distribution over $A_i$. For any \textit{strategy profile} $\sigma \in \Sigma = \{\Sigma_\pr \times \Sigma_\ps\}$ we use $u_i(\sigma) = u_i(\sigma_i, \sigma_{-i})$ as the expected outcome for player $i$, given the players follow strategy profile $\sigma$. A \textit{best response} (BR) of player $i$ to the opponent's strategy $\sigma_{-i}$ is a strategy $\sigma_i^{BR} \in BR_i(\sigma_{-i})$, where $u_i(\sigma_i^{BR}, \sigma_{-i}) \geq u_i(\sigma'_i, \sigma_{-i})$ for all $\sigma'_i \in \Sigma_i$. An $\epsilon$-\textit{best response} is $\sigma_i^{\epsilon BR} \in \epsilon BR_i(\sigma_{-i}), \epsilon > 0$, where $u_i(\sigma_i^{\epsilon BR}, \sigma_{-i}) + \epsilon \geq u_i(\sigma'_i, \sigma_{-i})$ for all $\sigma'_i \in \Sigma_i$. Given a normal-form game $G = (N,A,u)$, a tuple of mixed strategies $( \sigma_i^{NE}, \sigma_{-i}^{NE})$, $\sigma_{i}^{NE}\in\Sigma_{i}, \sigma_{-i}^{NE}\in\Sigma_{-i}$ is a \emph{Nash Equilibrium} if $\sigma_i^{NE}$ is an optimal strategy of player $i$ against strategy $\sigma_{-i}^{NE}$. Formally: 
$\sigma_i^{NE} \in BR(\sigma_{-i}^{NE})\quad\forall i\in \{\pr,\ps \}$

In many situations, the roles of the players are asymmetric. One player (leader - $\pr$) has the power to commit to a strategy, and the other player (follower - $\ps$) plays the best response. This model has many real-world applications~\cite{Tambe}; for example, the leader can correspond to a defense agency committing to a protocol to protect critical facilities. The common assumption in the literature is that the follower breaks ties in favor of the leader. Then, the concept is called a Strong Stackelberg Equilibrium (SSE). 

A leader's strategy $\sigma^{SSE}\in\Sigma_\pr$ is a \emph{Strong Stackelberg Equilibrium} if $\sigma_\pr$ is an optimal strategy of the leader given that the follower best-responds. Formally: 
$
\sigma_\pr^{SSE} = \argmax_{\sigma_\pr' \in \Sigma_\pr} u_\pr(\sigma_\pr',BR_\ps(\sigma_\pr')).
$
In zero-sum games, SSE is equivalent to NE~\cite{conitzer2006computing} and the expected utility is denoted \textit{value of the game}.

\subsection{Two-player extensive-form games}
A two-player extensive-form game (EFG) consist of a set of players $N = \{\pr,\ps,c\}$, where $c$ denotes the chance. $A$ is a finite set of all actions available in the game. $H \subset \{a_1 a_2 \cdots a_n \mid a_j \in A, n \in \mathbb{N}\}$ is the set of histories in the game. We assume that $H$ forms a non-empty finite prefix tree. We use $g \sqsubset h$ to denote that $h$ extends $g$. The \textit{root} of $H$ is the empty sequence $\emptyset$. The set of leaves of $H$ is denoted $Z$ and its elements $z$ are called \textit{terminal histories}. The histories not in Z are \textit{non-terminal histories}. By $A(h) = \{a \in A \mid ha \in H\}$ we denote the set of actions available at $h$. $P : H \setminus Z \to N$ is the \textit{player function} which returns who acts in a given history. Denoting $H_i = \{h \in H\setminus Z \mid P(h) = i\}$, we partition the histories as $H = H_\pr \cup H_\ps \cup H_c \cup Z$. $\sigma_c$ is the \textit{chance strategy} defined on $H_c$. For each $h \in H_c, \sigma_c(h)$ is a probability distribution over $A(h)$. Utility functions assign each player utility for each leaf node, $u_i : Z \to \mathbb{R}$.

The game is of \textit{imperfect information} if some actions or chance events are not fully observed by all players. The information structure is described by \textit{information sets} for each player $i$, which form a partition $\Iset_i$ of $H_i$. For any information set $I_i \in \Iset_i$, any two histories $h, h' \in I_i$ are indistinguishable to player $i$. Therefore $A(h) = A(h')$ whenever $h, h' \in I_i$. For $I_i \in \Iset_i$ we denote by $A(I_i)$ the set $A(h)$ and by $P(I_i)$ the player $P(h)$ for any $h \in I_i$.

A \textit{strategy} $\sigma_i \in \Sigma_i$ of player $i$ is a function that assigns a distribution over $A(I_i)$ to each $I_i \in \Iset_i$. A \textit{strategy profile} $\sigma = (\sigma_\pr, \sigma_\ps)$ consists of strategies for both players. $\pi^\sigma(h)$ is the probability of reaching $h$ if all players play according to $\sigma$. We can decompose $\pi^\sigma(h) = \prod_{i \in N}\pi^\sigma_i(h)$ into each player's contribution. Let $\pi^\sigma_{-i}$ be the product of all players' contributions except that of player $i$ (including chance). For $I_i \in \Iset_i$ define $\pi^\sigma(I_i) = \sum_{h \in I_i}\pi^\sigma(h)$, as the probability of reaching information set $I_i$ given all players play according to $\sigma$. $\pi_i^\sigma(I_i)$ and $\pi_{-i}^\sigma(I_i)$ are defined similarly. Finally, let $\pi^\sigma(h,z) = \frac{\pi^\sigma(z)}{\pi^\sigma(h)}$ if $h \sqsubset z$, and zero otherwise. $\pi^\sigma_i(h,z)$ and $\pi^\sigma_{-i}(h,z)$ are defined similarly. Using this notation, \textit{expected payoff} for player $i$ is $u_i(\sigma) = \sum_{z \in Z}u_i(z)\pi^\sigma(z)$. BR, NE and SSE are defined as in NFGs.

Define $u_i(\sigma, h)$ as an expected utility given that the history $h$ is reached and all players play according to $\sigma$. A \textit{counterfactual value} $v_i(\sigma,I)$ is the expected utility given that the information set $I$ is reached and all players play according to strategy $\sigma$ except player $i$, which plays to reach $I$. Formally, $v_i(\sigma,I) = \sum_{h \in I, z \in Z}\pi^\sigma_{-i}(h)\pi^\sigma(h,z)u_i(z)$. And similarly counterfactual value for playing action $a$ in information set $I$ is $v_i(\sigma,I, a) = \sum_{h \in I, z \in Z, ha \sqsubset z}\pi^\sigma_{-i}(ha)\pi^\sigma(ha,z)u_i(z)$.

We define $S_i$ as a set of sequences of actions only for player $i$. $inf_1(s_i), s_i \in S_i$ is the information set where last action of $s_i$ was executed and $seq_i(I), I \in \Iset_i$ is sequence of actions of player $i$ to information set $I$.

\subsection{Quantal response model of bounded rationality}

Fully rational players always select the utility-maximizing strategy, i.e., the best response. Relaxing this assumption leads to a ``statistical version'' of best response, which takes into account the inevitable error-proneness of humans and allows the players to make systematic errors \cite{mcfadden1976quantal, mckelvey1995quantal}. 

\begin{definition} Let $G = (N,A,u)$ be an NFG. Function $QR:\Sigma_{\pr}\rightarrow\Sigma_\ps$ is a quantal response function of player $\ps$ if probability of playing action $a$ monotonically increases as expected utility for $a$ increases. Quantal function QR is called canonical if for some real-valued function $q$:
\begin{equation}
QR(\sigma, a^k) = \frac{q(u_\ps(\sigma, a^k))}{\sum_{a^i\in A_\ps}q(u_\ps(\sigma, a^i))} \quad\forall\sigma\in\Sigma_{\pr}, a^k\in A_\ps.
\end{equation}
\end{definition}

Whenever $q$ is a strictly positive increasing function, the corresponding $QR$ is a valid quantal response function. Such functions $q$ are called \textit{generators} of canonical quantal functions. The most commonly used generator in the literature is the exponential (logit) function~\cite{mckelvey1995quantal} defined as $q(x) = e^{\lambda x}$ where $\lambda> 0$. $\lambda$ drives the model's rationality. The player behaves uniformly randomly for $\lambda\rightarrow0$, and becomes more rational as $\lambda\rightarrow\infty$. We denote a logit quantal function as LQR.

In EFGs, we assume the bounded-rational player plays based on a quantal function in every information set separately, according to the counterfactual values.

\begin{definition} Let $G$ be an EFG. Function $QR:\Sigma_{\pr}\rightarrow\Sigma_\ps$ is a canonical couterfactual quantal response function of player $\ps$ with generator $q$ if for a strategy $\sigma_{\pr}$ it produces strategy $\sigma_\ps$ such that in every information set $I \in \Iset_\ps$, for each action $a^k \in A(I)$ it holds that
\begin{equation}
QR(\sigma_{\pr}, I, a^k) = \frac{q(v_\ps(\sigma,I,a^k))}{\sum_{a^i \in A(I)}q(v_\ps(\sigma, I, a^i))},
    \label{eq:efg_lqr}
\end{equation}
where $QR(\sigma_\pr, I,a^k)$ is the probability of playing action $a^k$ in information set $I$ and $\sigma = (\sigma_\pr, \sigma_\ps)$.
% where $\sigma|_{I \to a}$ is strategy profile same as $\sigma$ except in information set $I$ we always play action $a$. 
% $\bar{u}_i(I,\sigma)$ is the expected utility in information set $I$ given $\sigma$ conditioned on reaching the information set $I$, defined as 
% \begin{equation}
%     \bar{u}_i(I,\sigma) = \sum_{h \in I, z \in Z}\frac{\pi^\sigma(z)u_i(z)}{\pi^\sigma(I)}
%     \label{eq:efg_uval}
% \end{equation}
\end{definition}

We denote the canonical counterfactual quantal response function with the logit generator \textit{counterfactual logit quantal response (CLQR)}. CLQR differs from the traditional definition of logit agent quantal response (LAQR)~\cite{mckelvey1998quantal} in using counterfactual values instead of expected utilities. The main advantage of CLQR over LAQR is that CLQR defines a valid quantal strategy even in information sets unreachable due to a strategy of the opponent, which is necessary for applying regret-minimization algorithms explained later.

Because the logit quantal function is the most well-studied function in the literature with several deployed applications~\cite{Pita08,delle2014game,fang2017paws}, we focus most of our analysis and experimental results on (C)LQR. Without a loss of generality, we assume the quantal player is always player $\ps$.

\subsection{Metrics for evaluating quality of strategy}
In a two-player zero-sum game, the \textit{exploitability} of a given strategy is defined as expected utility that a fully rational opponent can achieve above the value of the game. Formally, exploitability $\mathcal{E}(\sigma_i)$ of strategy $\sigma_i\in\Sigma_i$ is
$
    \mathcal{E}(\sigma_i) =  u_{-i}(\sigma_i, \sigma_{-i}) - u_{-i}(\sigma^{NE}), \quad \sigma_{-i} \in BR_{-i}(\sigma_i).
$

We also intend to measure how much we are able to exploit an opponent's bounded-rational behavior. For this purpose, we define \textit{gain} of a strategy against quantal response as an expected utility we receive above the value of the game. Formally, gain $\mathcal{G}(\sigma_i)$ of strategy $\sigma_i$ is defined as
$
\mathcal{G}(\sigma_i) = u_i(\sigma_i, QR(\sigma_i)) - u_{i}(\sigma^{NE}).
$

General-sum games do not have the property that all NEs have the same expected utility. Therefore, we simply measure expected utility against LQR and BR opponents there.
\section{One-Sided Quantal Solution Concepts}
\label{sec:concepts}
This section formally defines two one-sided bounded-rational equilibria, where one of the players is rational and the other is subrational -- a saddle-point-type equilibrium called Quantal Nash Equilibrium (QNE) and a leader-follower-type equilibrium called Quantal Stackelberg Equilibrium (QSE). We show that contrary to their fully-rational counterparts, QNE differs from QSE even in zero-sum games. Moreover, we show that computing QSE in extensive-form games is an NP-hard problem. Full proofs of all our theoretical results are provided in the appendix.
%In this section, we formally describe the properties of the quantal response. We show proofs regarding QRs, which are pretty good responses and QRs that can not be pretty good responses. Next, we describe solution concepts induced by the QR model.

\subsection{Quantal equilibria in normal-form games}
We first consider a variant of NE, in which one of the players plays a quantal response instead of the best response.

\begin{definition}
 Given a normal-form game $G = (N,A,u)$ and a quantal response function $QR$, a strategy profile $(\sigma_\pr^{QNE}, QR(\sigma_\pr^{QNE}))\in\Sigma$ describes a \emph{Quantal Nash Equilibrium} (QNE) if and only if $\sigma_{\pr}^{QNE}$ is a best response of player $\pr$ against quantal-responding player $\ps$. Formally: 
\begin{equation}
\sigma_{\pr}^{QNE} \in BR(QR(\sigma_{\pr}^{QNE})).
\end{equation}
\end{definition}

QNE can be seen as a concept between NE and Quantal Response Equilibrium (QRE)~\cite{mckelvey1995quantal}. While in NE, both players are fully rational, and in QRE, both players are bounded-rational, in QNE, one player is rational, and the other is bounded-rational.

\begin{theorem}
Computing a QNE strategy profile in two-player NFGs is a PPAD-hard problem.
\label{thm:one}
\end{theorem}
\begin{sproof}
We do a reduction from the problem of computing $\epsilon$-NE \cite{daskalakis2009complexity}. We derive an upper bound on a maximum distance between best response and logit quantal response, which goes to zero with $\lambda$ approaching infinity. For a given $\epsilon$, we find $\lambda$, such that QNE is $\epsilon$-NE. Detailed proofs of all theorems are provided in the appendix. 
\end{sproof}

% Detailed proofs are all in the appendix. 
QNE usually outperforms NE against LQR in practice as we show in the experiments. However, it cannot be guaranteed as stated in the Proposition~\ref{prop:badqne}.
\begin{proposition}
    For any $LQR$ function, there exists a zero-sum normal-form game $G = (N,A,u)$ with a unique NE~-~$(\sigma_\pr^{NE},\sigma_\ps^{NE})$ and QNE - $(\sigma_\pr^{QNE},QR(\sigma_\pr^{QNE}))$ such that $u_\pr(\sigma_\pr^{NE},QR(\sigma_\pr^{NE})) > u_\pr(\sigma_\pr^{QNE},QR(\sigma_\pr^{QNE}))$.%\footnote{Full proofs of all propositions are in the \href{\appendixurl}{\color{\appendixcolor}\underline{appendix}}.}
    \label{prop:badqne}
\end{proposition}
The second solution concept is a variant of SSE in situations, when the follower is bounded-rational.

\begin{definition} Given a normal-form game $G = (N,A,u)$ and a quantal response function $QR$, a mixed strategy $\sigma_\pr^{QSE}\in\Sigma_\pr$ describes a \emph{Quantal Stackleberg Equilibrium} (QSE) if and only if
\begin{equation}
\label{eq:qse:primal}
\sigma_\pr^{QSE} = \argmax_{\sigma_\pr\in\Sigma_\pr} u_\pr(\sigma_\pr, QR(\sigma_\pr)).
\end{equation}
\end{definition}

In QSE, player $\pr$ is fully rational and commits to a strategy that maximizes her payoff given that player $\ps$ observes the strategy and then responds according to her quantal function. This is a standard assumption, and even in problems where the strategy is not known in advance, it can be learned by playing or observing. QSE always exists because all utilities are finite, the game has a finite number of actions, player $\pr$ utilities are continuous on her strategy simplex, and the maximum is hence always reached.

\begin{observation}
\label{obs:qse_nfg_mp}
Let $G$ be a normal-form game and a $q$ be a generator of a canonical quantal function. Then QSE of $G$ can be formulated as a non-convex mathematical program:
\begin{flalign}
\begin{aligned}
\label{eq:qse:zs}
\max_{\sigma_\pr\in\Sigma_\pr}&\frac{\sum_{a_\ps\in A_\ps}u_\pr(\sigma_\pr, a_\ps)q(u_\ps(\sigma_\pr, \pi_\ps))}{\sum_{a_\ps\in A_\ps}q(u_\ps(\sigma_\pr,a_\ps))}.
%1 &= \sum_{a\in\calA_l}\delta[a]
\end{aligned}
\end{flalign}
\end{observation}

\begin{figure}
\begin{minipage}{0.49\linewidth}
\centering
\includegraphics[width=\linewidth]{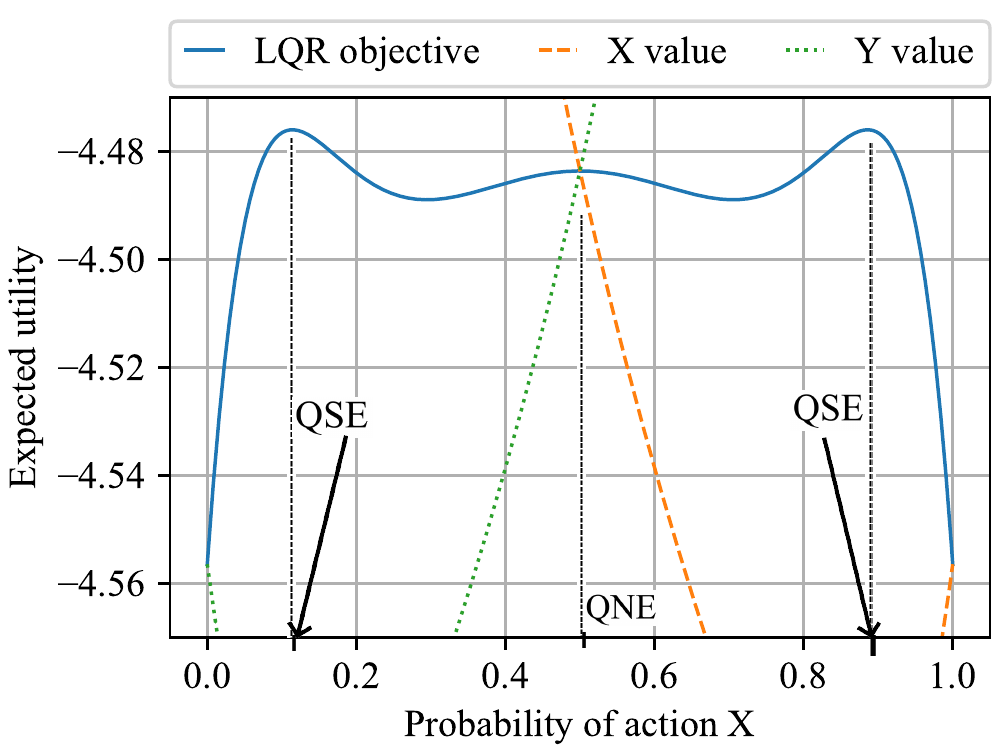}
\end{minipage}
\begin{minipage}{0.49\linewidth}
\scriptsize
 \centering
 \begin{tabular}{ccccc}
  & \multicolumn{4}{c}{Game 1} \\
  & A & B & C & D\\ \cline{2-5}
 X & \multicolumn{1}{|c|}{-4} & \multicolumn{1}{|c|}{-5}  & \multicolumn{1}{|c|}{8} & \multicolumn{1}{|c|}{-4} \\ \cline{2-5}
 Y & \multicolumn{1}{|c|}{-5} & \multicolumn{1}{|c|}{-4} & \multicolumn{1}{|c|}{-4} & \multicolumn{1}{|c|}{8} \\ \cline{2-5}
 \end{tabular}

\vspace{0.2cm}
\begin{tabular}{ccc}
  &  \multicolumn{2}{c}{Game 2} \\
  & A & B \\ \cline{2-3}
 X & \multicolumn{1}{|c|}{-2} & \multicolumn{1}{|c|}{8} \\ \cline{2-3}
 Y & \multicolumn{1}{|c|}{-2.2} & \multicolumn{1}{|c|}{-2.5} \\ \cline{2-3}
\end{tabular}
\end{minipage}
\caption{(Left) An example of the expected utility against LQR in Game 1. The X-axis shows the strategy of the rational player, and the Y-axis shows the expected utility. The X- and Y-value curves show the utility for playing the corresponding action given the opponent's strategy is a response to the strategy on X-axis.  A detailed description is in Example~\ref{QSE_example}. (Right) An example of two normal-form games. Each row of the depicted matrices is labeled by a first player strategy, while the second player's strategy labels every column. The numbers in the matrices denote the utilities of the first player. We assume the row player is $\pr$.
%Example of NFGs where letters in tables are actions in the game and numbers in the table are utilities for player 1 in the game. The row player is the rational player, and the left plot shows expected utilities from Game~1. X and Y values show the utility for playing the corresponding action given the opponent strategy is fixed.
}
\label{LQRobjgraph}
\end{figure}

\begin{example}
In Figure~\ref{LQRobjgraph} we depict a utility function of $\pr$ in Game 1 against LQR with $\lambda = 0.92$. As we show, both actions have the same expected utility in the QNE. Therefore, it is a best response for Player $\pr$, and she has no incentive to deviate. However, the QNE does not reach the maximal expected utility, which is achieved in two global extremes, both being the QSE. Note that even such a small game like Game 1 can have multiple local extremes: in this case 3.
 \label{QSE_example}
\end{example}

% Example~\ref{QSE_example} shows that finding QSE is a non-concave problem even in zero-sum NFGs, and it can have multiple global solutions. %In~\cite{cerny2020}, the authors empirically observed that a number of local optima of the objective function of QSE can grow at least linearly in the number of actions of the follower.

Example~\ref{QSE_example} shows that finding QSE is a non-concave problem even in zero-sum NFGs, and it can have multiple global solutions. Moreover, facing a bounded-rational opponent may change the relationship between NE and SSE. They are no longer interchangeable in zero-sum games, and QSE may use strictly dominated actions, e.g in Game 2 from Figure~\ref{LQRobjgraph}.

% \begin{observation}
% QNE and QSE are not interchangeable even in zero-sum games; hence, QNE does not guarantee optimal value against LQR. Interestingly, the rational player can use strictly dominated actions in the QSE strategy.
% \end{observation}

\subsection{Quantal equilibria in extensive-form games}
% define QNE
% define QSE
% obs: gen-sum QSE is NP-H

In EFGs, QNE and QSE are defined in the same manner as in NFGs. However, instead of the normal-form quantal response, the second player acts according to the counterfactual quantal response. QSE in EFGs can be computed by a mathematical program provided in the appendix. The natural formulation of the program is non-linear with non-convex constraints indicating the problem is hard. We show that the problem is indeed NP-hard, even in zero-sum games.

%An optimal strategy against CLQR is defined in the same way as in NFGs. We maximize the objective function, which only depends on the strategy of player 1. We call the strategy profile $\sigma = (\sigma_i^{CQSE},CLQR(\sigma_i^{CQSE}))$ which maximizes CLQR objective function counterfactual quantal Stackelberg equilibrium (CQSE). We can formulate the problem of finding CQSE as a sequence form program, and the solution of this program is the optimal strategy against CLQR. $S_i$ is a set of sequences of actions only for player $i$. $inf_1(s_i), s_i \in S_i$ is the information set where last action of $s_i$ was executed and $seq_i(I), I \in \Iset$ is sequence of player $i$ to information set $I$.  The program is as follows

\begin{theorem}
Let $G$ be a two-player imperfect-information EFG with perfect recall and $QR$ be a quantal response function. Computing a QSE in $G$ is NP-hard if one of the following holds:
(1) $G$ is zero-sum and $QR$ is generated by a logit generator $q(x) = exp(\lambda x)$, $\lambda>0$;
or (2) $G$ is general-sum.
\label{thm:np}
\end{theorem}
\begin{sproof}
% We reduce from the partition problem. The key part of the constructed EFG is zero-sum. For each item of the partition problem, the leader chooses an action that places the item to one or the other subset. The follower has two actions; each gives the leader a reward of the sum of items in one subset. If the sums are different, the follower chooses the lower one. If they are the same, the follower chooses both of them uniformly, which maximizes the leader's payoff.

% A complication is that the leader could split each item in half by playing uniformly. This is prevented by combining the leader's actions for placing an item with an action in a separate game with two symmetric QSEs. Such a game is the collaborative coordination game in the non-zero-sum case and a game similar to Figure~\ref{LQRobjgraph} in the zero-sum case.
We reduce from the partition problem. The key subtree of the constructed EFG is zero-sum. For each item of the multiset, the leader chooses an action that places the item in the first or second subset. The follower has two actions; each gives the leader a reward equal to the sum of items in the corresponding subset. If the sums are not equal, the follower chooses the lower one because of the zero-sum assumption. Otherwise, the follower chooses both actions with the same probability, maximizing the leader's payoff.

A complication is that the leader could split each item in half by playing uniformly. This is prevented by combining the leader's actions for placing an item with an action in a separate game with two symmetric QSEs. We use a collaborative coordination game in the non-zero-sum case and a game similar to game in Figure~\ref{LQRobjgraph} in the zero-sum case.
\end{sproof}

The proof of the non-zero-sum part of Theorem~\ref{thm:np} relies only on the assumption that the follower plays action with higher reward with higher probability. This also holds for a rational player; hence, the theorem provides an independent, simpler, and more general proof of NP-hardness of computing Stackelberg equilibria in EFGs, which unlike \cite{letchford2010computing} does not require the tie-breaking rule.

\section{Computing One-Sided Quantal Equilibria}

This section describes various algorithms and heuristics for computing one-sided quantal equilibria introduced in the previous section. In the first part, we focus on QNE, and based on an empirical evaluation; we claim that regret-minimization algorithms converge to QNE in both NFGs and EFGs. The second part then discusses gradient-based algorithms for computing QSE and analyses cases when regret minimization methods will or will not converge to QSE. 

%We describe all the algorithms we proposed and tested in our work. Starting with well-known CFR-f and gradient ascent and finishing with our proposed heuristic algorithms.

\subsection{Algorithms for computing QNE}
CFR \cite{zinkevich2008regret} is a state-of-the-art algorithm for approximating NE in EFGs. CFR is a form of regret matching \cite{hart2000simple} and uses iterated self play to minimize regret at each information set independently. CFR-f \cite{davis2014using} is a modification capable of computing strategy against some opponent models. In each iteration, it performs a CFR update for one player and computes the response for the other player. We use CFR-f with a quantal response and call it CFR-QR. We initialize the rational player's strategy as uniform and compute the quantal response against it. Then, in each iteration, we update the regrets for the rational player, calculate the corresponding strategy, and compute a quantal response to this new strategy again. In normal-form games, we use the same approach with simple regret matching (RM-QR).

% The baseline approach for computing an exact NE is via mathematical programming~\cite{von2002computing,bosansky13}. Scaling up to even larger games requires approximating NE, for which the leading method is \textit{Regret Minimization (RM)}. RM algorithms are iterative methods that, in every iteration, update the players' strategies to minimize a weighted sum of regret. The average strategies are then guaranteed to approach NE in zero-sum games~\cite{blum2007learning}. We consider a variant of the original RM dynamics, in which the rational player performs the RM update while the bounded-rational player responds to the rational player's current strategy according to her quantal function (RMQR). In EFGs, this approach is called CFR-f \cite{davis2014using}. 

% To tackle the scalability issue of solving the problem optimally, we again use regret minimization. In EFG with imperfect information, we use CFR-f \cite{davis2014using} with CLQR. In each iteration, we perform regret matching for the rational player, and for the opponents, we compute the corresponding CLQR to the rational player's current strategy. We call the algorithm CFR-QR.

% \cite{2007-nips-cfr}.

\begin{conjecture}
 (a) In NFGs, RM-QR converges to QNE; while (b) in EFGs, CFR-QR converges to QNE.
 \label{convergence}
\end{conjecture}
% \begin{proof}
% Define loss as $l(\sigma_1) = -u_1(\sigma_1, QR(\sigma_1))$, then by Lemma 2 of~\cite{lockhart2019computing} it holds that $l(\sigma^*) - \min_{\sigma\in\Sigma_1}l(\sigma) \leq \epsilon / T$ for a RM algorithm with regret smaller than $\epsilon$ after T iterations. Hence $-u_1(\sigma^*, QR(\sigma^*)) - \max_{\sigma\in\Sigma_1}u_1(\sigma, QR(\sigma)) \leq \epsilon / T$
% \end{proof}

Conjecture~\ref{convergence} is based on empirical evaluation on more than $2\times 10^4$ games. In each game, the resulting strategy of player $\pr$ was $\epsilon$-BR to the quantal response of the opponent with epsilon lower than $10^{-6}$ after less than $10^5$ iterations.

% We extend \textbf{Conjecture~\ref{lemma_qne}} to CFR-QR in EFGs. It is based on empirical data from EFGs using over 5 thousand games up to a size of 50 thousand nodes. The resulting epsilon was less than $10^{-6}$ after less than 100 thousand iterations.

% \subsubsection{Combination Heuristics}
\label{efg:heu}

QNE provides strategies exploiting a quantal opponent well, but performance is at the cost of substantial exploitability. We propose two heuristics that address both gain and exploitability simultaneously. The first one is to play a convex combination of QNE and NE strategy. We call this heuristic algorithm \textbf{COMB}. We aim to find a parameter $\alpha$ of the combination that maximizes the utility against LQR. However, choosing the correct $\alpha$ is, in general, a non-convex, non-linear problem. We search for the best $\alpha$ by sampling possible $\alpha$s and choosing the one with the best utility. The time required to compute one combination's value is similar to the time required to perform one iteration of the RM-QR algorithm. Sampling the $\alpha$s and checking the utility hence does not affect the scalability of COMB. The gain is also guaranteed to be greater or equal to the gain of the NE strategy, and as we show in the results, some combinations achieve higher gains than both the QNE and the NE strategies.

The second heuristic uses a restricted response approach~\cite{johanson2008computing}, and we call it \textbf{restricted quantal response (RQR)}. The key idea is that during the regret minimization, we set probability $p$, such that in each iteration, the opponent updates her strategy using (i) LQR with probability $p$ and (ii) BR otherwise. We aim to choose the parameter $p$ such that it maximizes the expected payoff. Using sampling as in COMB is not possible, since each sample requires to rerun the whole RM. To avoid the expensive computation, we start with $p=0.5$ and update the value during the iterations. In each iteration, we approximate the gradient of gain with respect to $p$ based on a change in the value after both the LQR and the BR iteration. We move the value of $p$ in the gradient's approximated direction with a step size that decreases after each iteration. However, the strategies do change tremendously with $p$, and the algorithm would require many iterations to produce a meaningful average strategy. Therefore, after a few thousands of iterations, we fix the parameter $p$ and perform a clean second run, with $p$ fixed from the first run. Similarly to COMB, RQR achieves higher gains than both the QNE and the NE and performs exceptionally well in terms of exploitability with gains comparable to COMB.

We adapted both algorithms from NFGs also to EFGs. The COMB heuristic requires to compute a convex combination of strategies, which is not straightforward in EFGs. Let $p$ be a combination coefficient and $\sigma_i^1$, $\sigma_i^2 \in \Sigma_i$ be two different strategies for the player $i$. The convex combination of the strategies is a strategy $\sigma_i^3 \in \Sigma$ computed for each information set $I_i \in \Iset_i$ and action $a \in A(I_i)$ as follows:

\label{convex}
\begin{align}
    \sigma_i^3&(I_i)(a) = \\&\frac{\pi_i^{\sigma^1}(I_i)\sigma_i^1(I_i)(a)p + \pi_i^{\sigma^2}(I_i)\sigma_i^2(I_i)(a)(1-p)}{\pi_i^{\sigma^1}(I_i)p + \pi_i^{\sigma^2}(I_i)(1-p)} \nonumber
\end{align}

We search for a value of $p$ that maximizes the gain, and we call this approach the counterfactual COMB. Contrary to COMB, the RQR can be directly applied to EFGs. The idea is the same, but instead of regret matching, we use CFR. We call this heuristic algorithm the counterfactual RQR.

\subsection{Algorithms for computing QSE}

In general, the mathematical programs describing the QSE in NFGs and EFGs are non-concave, non-linear problems. We use the gradient ascent (GA) methods~\cite{boyd2004convex} to find these programs' local optima. In case a program's formulation is concave, the GA will reach a global optimum. However, both formulations of QSE contain a fractional part, corresponding to a definition of the follower's canonical quantal function. Because concavity is not preserved under division, accessing conditions of the concavity of these programs is difficult. We construct provably globally optimal algorithms for QSE in our concurrent papers~\cite{cerny2020qseefg,cerny2020}. The GA performs well on small games, but it does not scale at all even for moderately sized games, as we show in the experiments.

% QSE is the optimal solution of mathematical programs formulated in \ref{sec:concepts}. However, the program for NFGs as non-linear non-convex objective and for EFGs, all constraints from Equation\ref{non_conv} are non-linear and non-convex. We use gradient ascent (GA) on the programs in order to find at least local optima. The GA performs very well on smaller games but does not scale even for moderate games, as we show in Section~\ref{eval}.

Because QSE and QNE are usually non-equivalent concepts even in zero-sum games (see Figure 1), the regret-minimization algorithms will not converge to QSE. However, in case a quantal function satisfies the so-called \textit{pretty-good-response} condition, the algorithm converges to a strategy of the leader exploiting the follower the most \cite{davis2014using}. We show that a class of simple (i.e., attaining only a finite number of values) quantal functions satisfy a pretty-good-responses condition. 

\begin{proposition}
Let $G = (N,A,u)$ be a zero-sum NFG, $QR$ a quantal response function of the follower, which depends only on the ordering of expected utilities of individual actions. Then the RM-QR algorithm converges to QSE.
\label{prop:RMQRconv}
\end{proposition}

An example of a simple quantal function depending only on the ordering of expected utilities is, e.g., a function assigning probability $0.5$ to the actions with the highest expected utility, probability $0.3$ to the action with the second-highest utility and probabilities $0.2/(\mid A_2\mid -2)$ to all remaining actions. Note that the class of quantal functions satisfying the conditions of pretty-good-responses still takes into account the strategy of the opponent (i.e., the responses are not static), but it is limited. In general, quantal functions do not satisfy the condition of pretty-good-responses.

\begin{proposition}
Let $QR$ be a canonical quantal function with a strictly monotonically increasing generator $q$. Then $QR$ is not a pretty-good-response.
\label{prop:notPGR}
\end{proposition}

\section{Experimental Evaluation}
\label{eval}
The experimental evaluation aims to compare solutions of our proposed algorithm RQR with QNE strategies computed by RM-QR for NFGs and CFR-QR for EFGs. As baselines, we use (i) Nash equilibrium (NASH) strategies, (ii) a best convex combination of NASH and QNE denoted COMB, and (iii) an approximation of QSE computed by gradient ascent (GA), initialized by NASH. We use regret matching+ in the regret-based algorithms.
We focus mainly on zero-sum games, because they allow for a more straightforward interpretation of the trade-offs between gain and exploitability. Still, we also provide results on general-sum NFGs. Finally, we show that the performance of RQR is stable over different rationality values and analyze the EFG algorithms more closely on well-known Leduc Hold'em game. The experimental setup and all the domains are described in the appendix. The implementation is available \href{https://gitlab.fel.cvut.cz/milecdav/aaai_qne_code.git}{\color{ntublue}here}.

\subsection{Scalability}
\begin{figure}
\centering
\includegraphics[width=\linewidth]{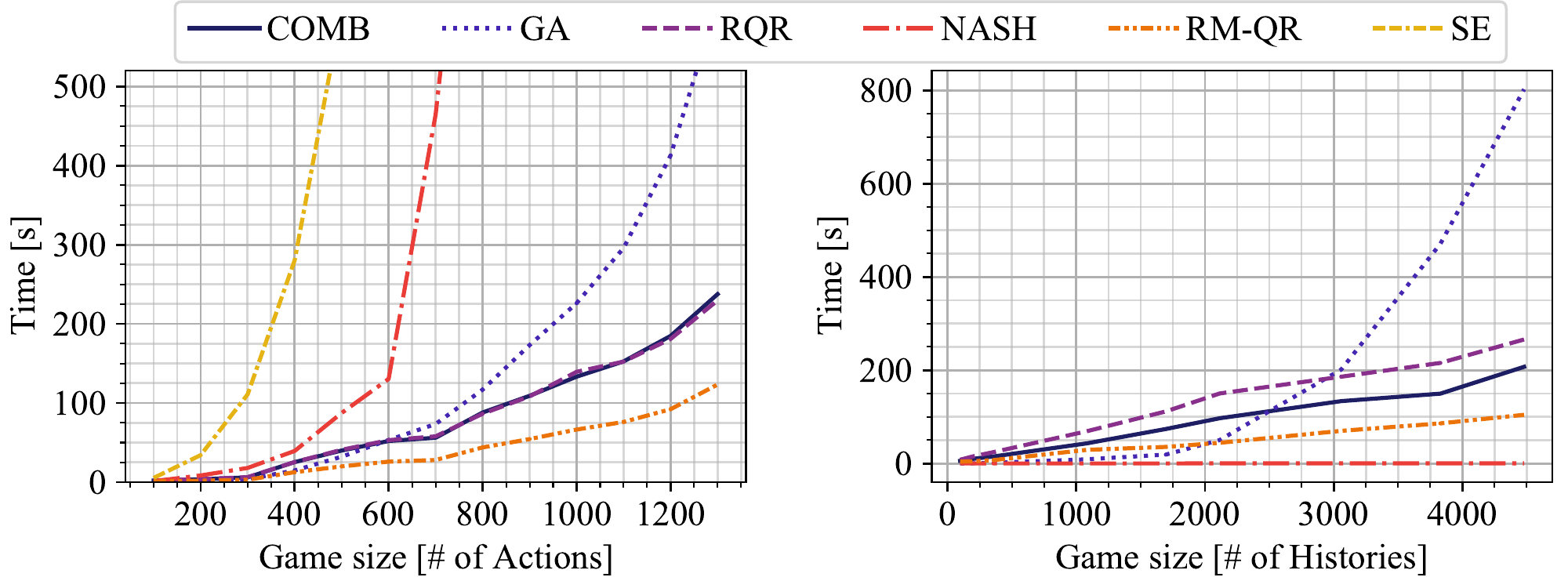}
\caption{Running time comparison of COMB, GA, RQR, NASH, SE and QNE on (left) square general-sum NFGs and (right) zero-sum EFGs. %NFGs are square, and size is the number of actions for each player.
}
\label{scaling_nfg}
\end{figure}
The first experiment shows the difference in runtimes of GA and regret-minimization approaches. In NFGs, we use random square zero-sum games as an evaluation domain, and the runtimes are averaged over 1000 games per game size. In EFGs, the random generation procedure does not guarantee games with the same number of histories, so we cluster games with a similar size together, and report runtimes averaged over the clusters. The results on the right of Figure~\ref{scaling_nfg} show that regret minimization approaches scale significantly better -- the tendency is very similar in both NFGs and EFGs, and we show the results for NFGs in the appendix.

We report scalability in general-sum games on the left in Figure~\ref{scaling_nfg}. We generated 100 games of Grab the Dollar, Majority Voting, Travelers Dilemma, and War of Attrition with an increasing number of actions for both players and also 100 random general-sum NFGs of the same size. Detailed game description is in the appendix. In the rest of the experiments, we use sets of 1000 games with 100 actions for each class. We use a MILP formulation to compute the NE \cite{sandholm2005mixed} and solve for SE using multiple linear programs \cite{conitzer2006computing}. The performance of GA against CFR-based algorithm is similar to the zero-sum case, and the only difference is in NE and SE, which are even less scalable than GA.

\subsection{Gain comparison}
Now we turn to a comparison of gains of all algorithms in NFGs and EFGs. We report averages with standard errors for zero-sum games in Figure~\ref{zs_nfg_perf} and general-sum games in Figure~\ref{nfg_perf} (left). We use the NE strategy as a baseline, but as different NE strategies can achieve different gains against the subrational opponent, we try to select the best NE strategy. To achieve this, we first compute a feasible NE. Then we run gradient ascent constrained to the set of NE, optimizing the expected value. We aim to show that RQR performs even better than an optimized NE. Moreover, also COMB strategies outperform the best NE, despite COMB using the (possibly suboptimal) NE strategy computed by CFR.

The results show that GA for QSE is the best approach in terms of gain in zero-sum and general-sum games if we ignore scalability issues. The scalable heuristic approaches also achieve significantly higher gain than both the NE baseline and competing QNE in both zero-sum and general-sum games. On top of that, we show that in general-sum games, in all games except one, the heuristic approaches perform as well as or better than SE.  This indicates that they are useful in practice even in general-sum settings.

\begin{figure}
\includegraphics[width=0.48\linewidth]{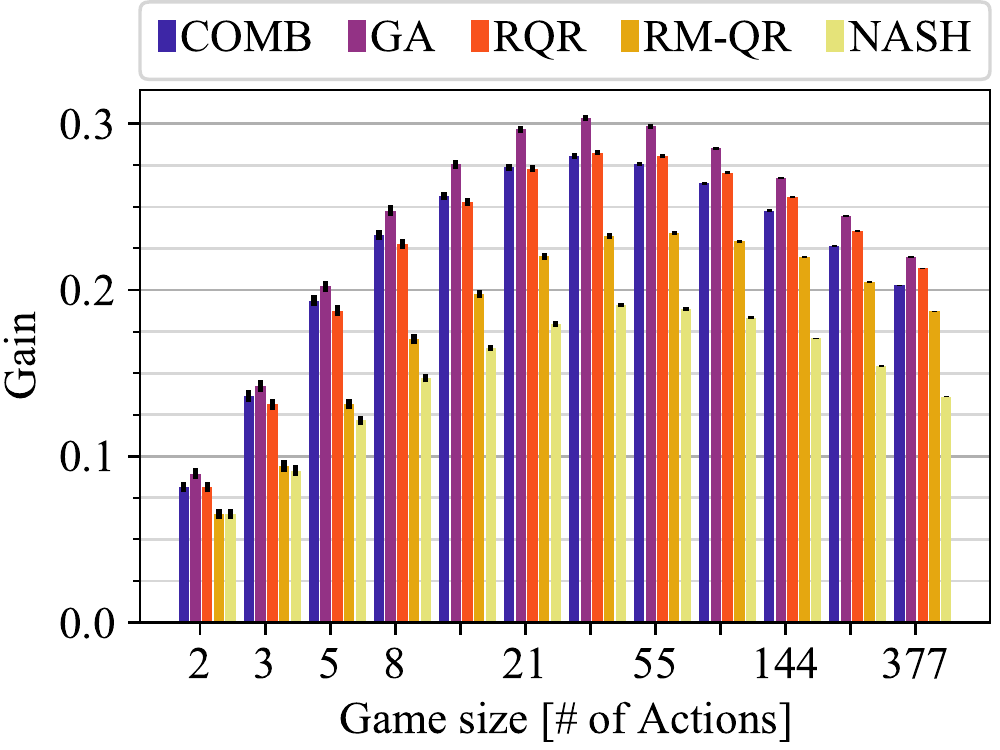}
\hspace{0.02\linewidth}
\includegraphics[width=0.48\linewidth]{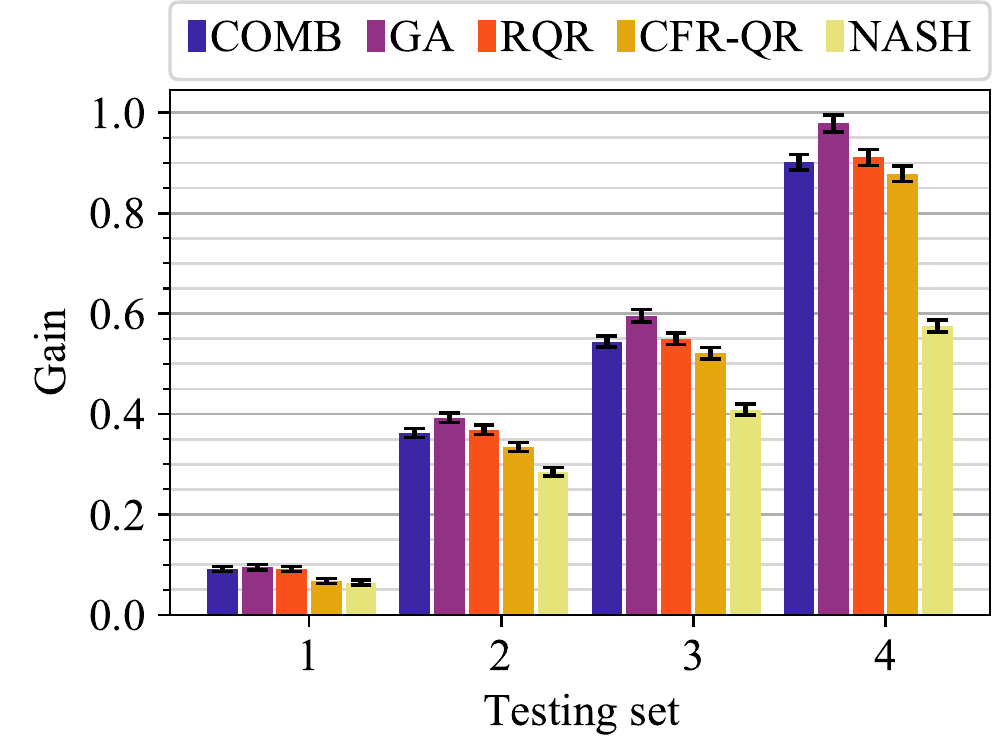}
\caption{Gain comparison of GA, Nash(SE), QNE, RQR and COMB in (left) random square zero-sum NFGs, and (right) random zero-sum EFGs.}
\label{zs_nfg_perf}
\end{figure}

\begin{figure}
\includegraphics[width=0.48\linewidth]{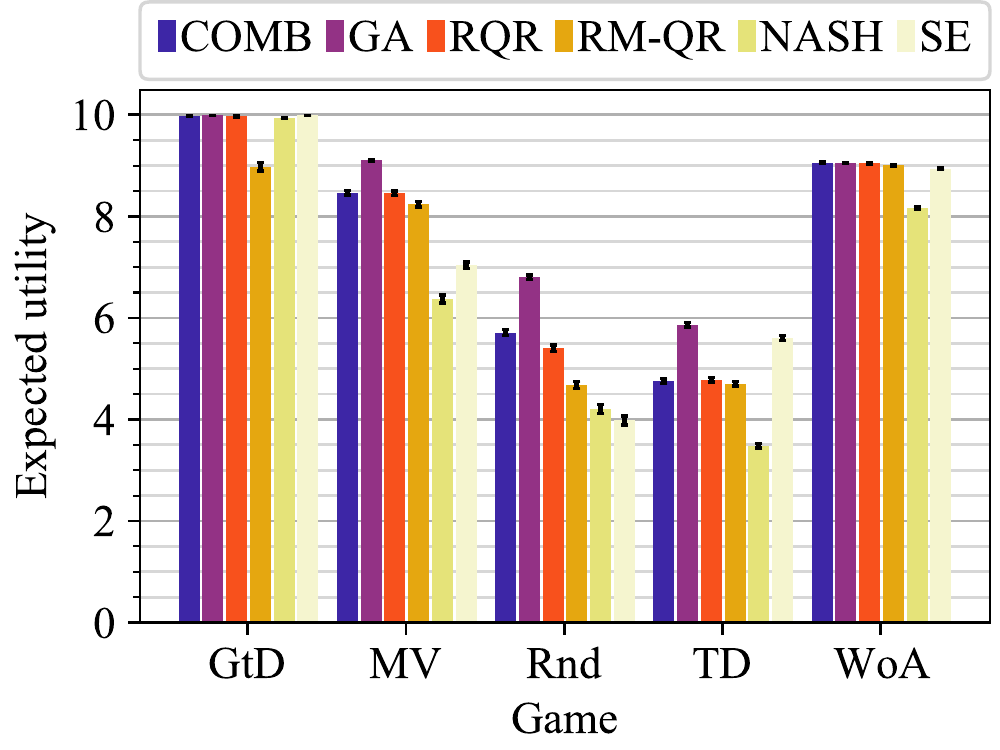}
\hspace{0.02\linewidth}
\includegraphics[width=0.48\linewidth]{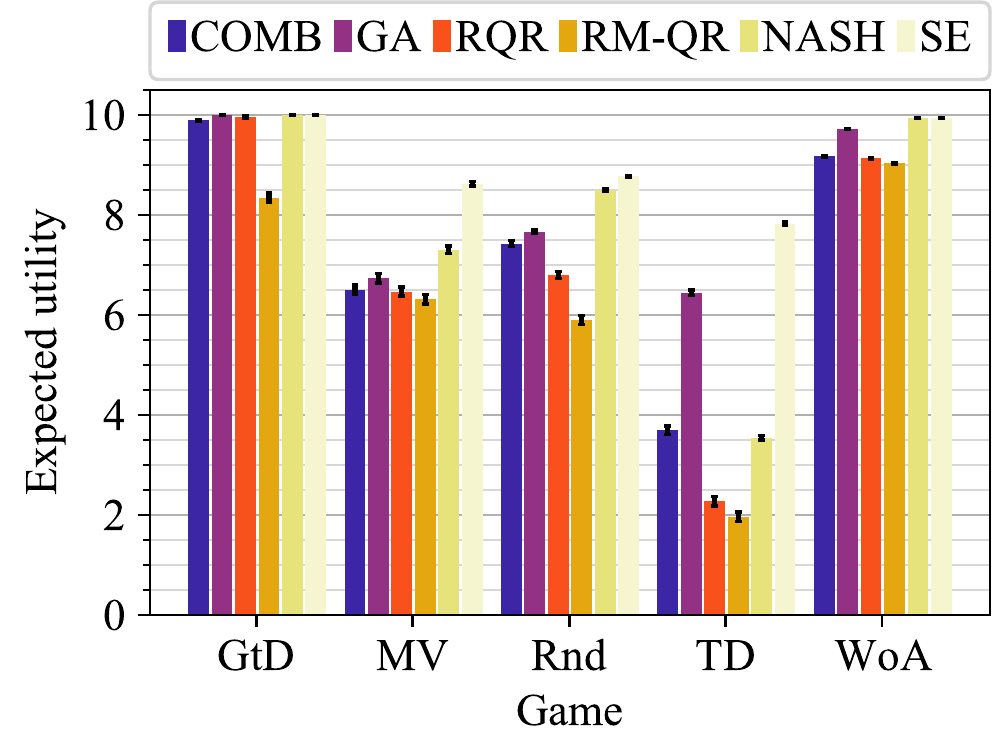}
\caption{Gain and robustness comparison of GA, Nash, Strong Stackelberg (SE), QNE, COMB and RQR in general sum NFGs -- the expected utility (left) against LQR and (right) against BR that maximizes leader's utility.}
\label{nfg_perf}
\end{figure}

\subsection{Robustness comparison}
In this work, we are concerned primarily with increasing gain. However, the higher gain might come at the expense of robustness--the quality of strategies might degrade if the model of the opponent is incorrect. Therefore, we study also (i) the exploitability of the solutions in zero-sum games and (ii) expected utility against the best response that breaks ties in our favor in general-sum games. Both correspond to performance against a perfectly rational selfish opponent.

First, we report the mean exploitability in zero-sum games in Figure~\ref{efg_perf}.Exploitability of NE is zero, so it is not included. We show that QNE is highly exploitable in both NFGs and EFGs. COMB and GA perform similarly, and RQR has significantly lower exploitability compared to other modeling approaches. Second, we depict the results in general-sum games on the right in Figure~\ref{nfg_perf}. By definition, SE is the optimal strategy and provides an upper bound on achievable value. Unlike in zero-sum games, GA outperforms CFR-based approaches even against the rational opponent. Our heuristic approaches are not as good as entirely rational solution concepts, but they always perform better than QNE. %Therefore even though, in some cases, rational concepts might be better when we choose to pick opponent modeling algorithms on large games where we can not compute GA, the heuristics are still better than QNE.

\begin{figure}
\includegraphics[width=0.48\linewidth]{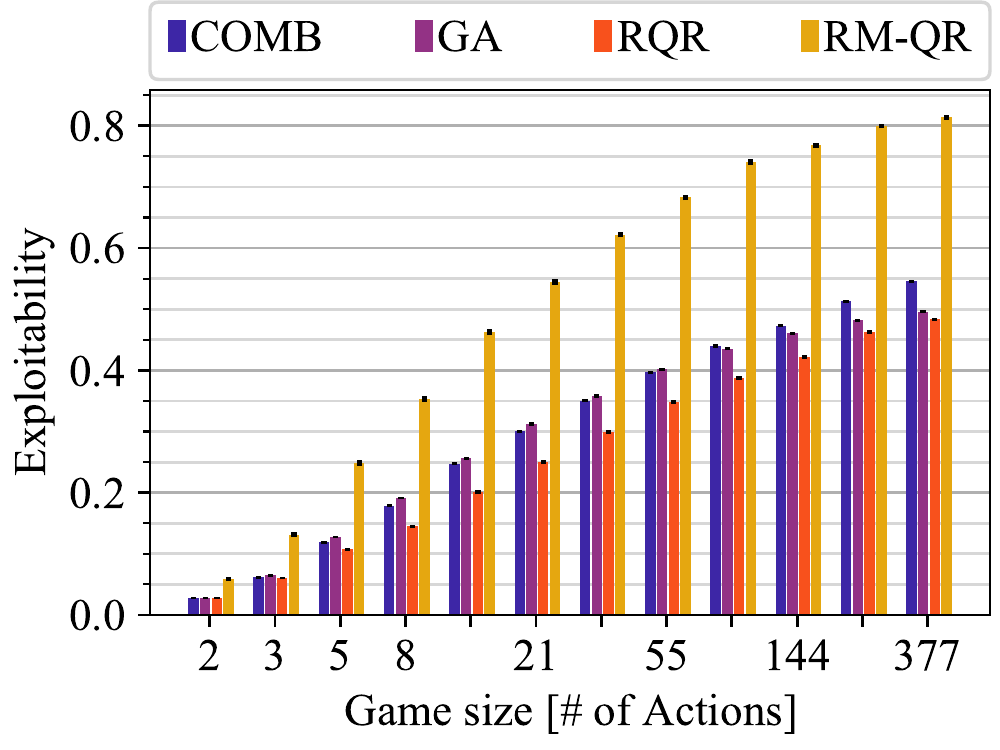}
\hspace{0.02\linewidth}
\includegraphics[width=0.48\linewidth]{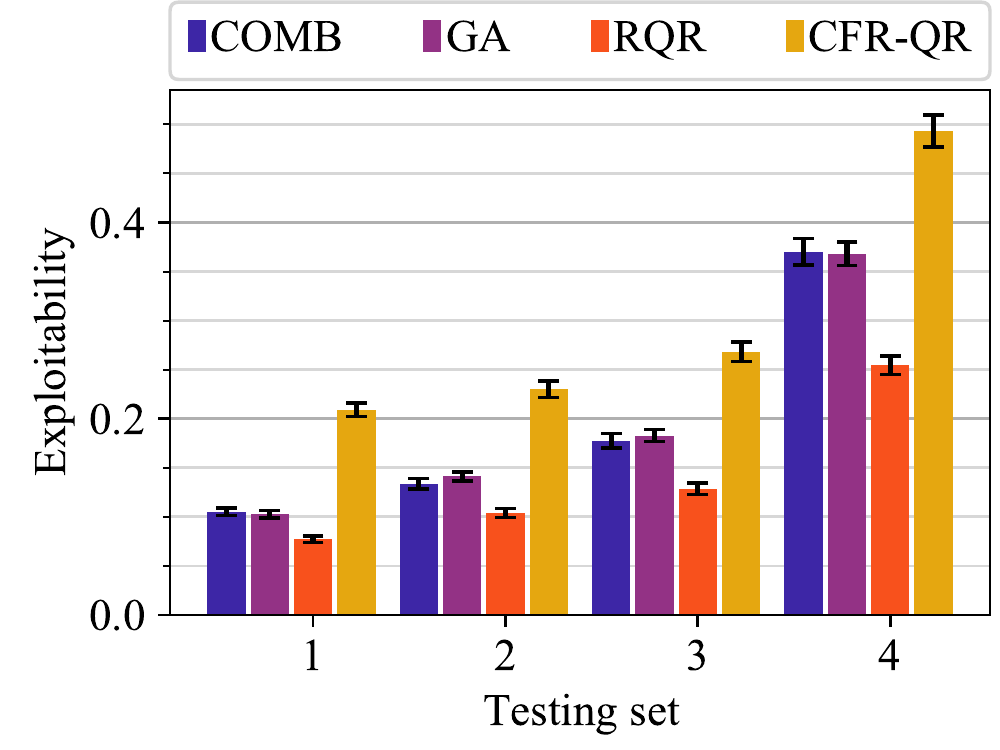}
\caption{Robustness comparison of GA, QNE, COMB and RQR in (left) random square zero-sum NFGs, and (right) random zero-sum EFGs ($\mathcal{E}(Nash)=0$).}
\label{efg_perf}
\end{figure}

\subsection{Different rationality}
In the fourth experiment, we access the algorithms' performance against opponents with varying rationality parameter $\lambda$ in the logit function. For $\lambda\in\{0,\dots100\}$ we report the expected utility on the left in Figure~\ref{leduc_perf}. For smaller values of $\lambda$ (i.e., lower rationality), RQR performs similarly to GA and QNE, but it achieves lower exploitability. As rationality increases, the gain of RQR is found between GA and QNE, while having the lowest exploitability. For all values of $\lambda$, both QNE and RQR report higher gain than NASH. We do not include COMB in the figure for the sake of better readability as it achieves similar results to RQR.

\subsection{Standard EFG benchmarks}
\paragraph{Poker.}
Poker is a standard evaluation domain, and continual resolving was demonstrated to perform extremely well on it \cite{moravvcik2017deepstack}. We tested our approaches on two poker variants: one-card poker and Leduc Hold'em. We used $\lambda = 2$ because for $\lambda = 1$, QNE is equal to QSE. We report the values achieved in Leduc Hold'em on the right in Figure~\ref{leduc_perf}. The horizontal lines correspond to NE and GA strategies, as they do not depend on $p$. The heuristic strategies are reported for different $p$ values. The leftmost point corresponds to the CFR-BR strategy and rightmost to the QNE strategy. The experiment shows that RQR performs very well for poker games as it gets close to the GA while running significantly faster. Furthermore, the strategy computed by RQR is much less exploitable consistently throughout various $\lambda$ values. This suggests that the restricted response can be successfully applied not only against strategies independent of the opponent as in \cite{johanson2008computing}, but also against adapting opponents. We observe similar performance also in the one-card poker and report the results in the appendix.

\paragraph{Goofspiel 7.} %We say our approach is scalable, but so far, we demonstrate the results on smaller games. 
We demonstrate our approach on Goofspiel 7, a game with almost 100 million histories to show a practical scalability. While CFR-QR, RQR, and CFR were able to compute a strategy, the games of this size are beyond the computing abilities of GA and memory requirements of COMB. CFR-QR has exploitability 4.045 and gains 2.357, RQR has exploitability 3.849 and gains 2.412, and CFR gains 1.191 with exploitability 0.115. RQR hence performs the best in terms of gain and outperforms CFR-QR in exploitability. All algorithms used 1000 iterations.

% \begin{table}
% \begin{tabular}{ccc}
%   &  \multicolumn{2}{c}{GS7} \\
%   & Exploitability & Gain \\ \cline{2-3}
%  CFR-QR & \multicolumn{1}{|c|}{4.04515} & \multicolumn{1}{|c|}{2.3572} \\ \cline{2-3}
%  CFR & \multicolumn{1}{|c|}{$\approx$0} & \multicolumn{1}{|c|}{1.19116} \\ \cline{2-3}
%  RQR & \multicolumn{1}{|c|}{3.84945} & \multicolumn{1}{|c|}{2.41211} \\ \cline{2-3}
% \end{tabular}
% \caption{Exploitability and Gain of different solution concepts on Goofspiel 7.}
% \label{gs}
% \end{table}

\begin{figure}
\includegraphics[width=0.48\linewidth]{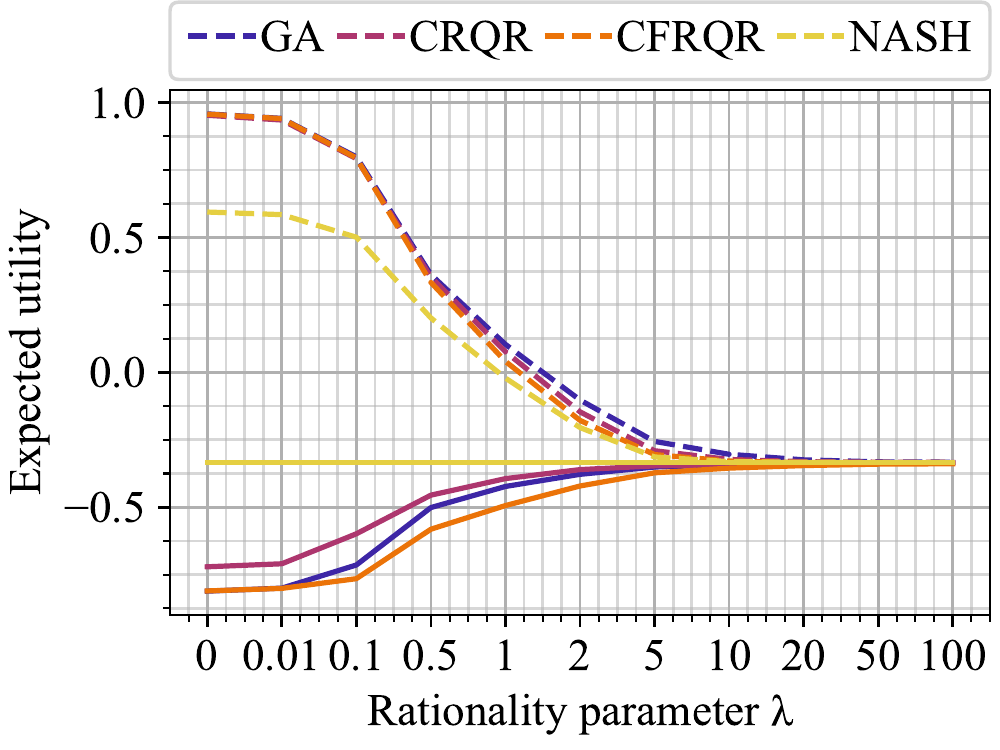}
\hspace{0.02\linewidth}
\includegraphics[width=0.48\linewidth]{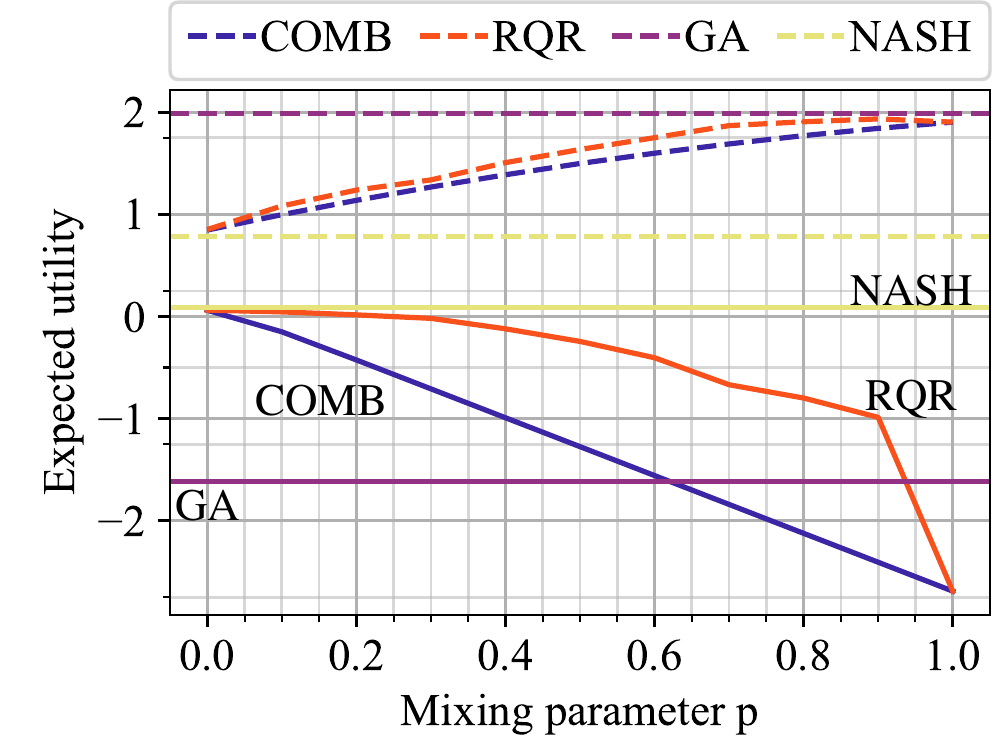}
\caption{(Left) Mean expected utility against CLQR (dashed) and BR (solid) over 100 games from set 2 of EFGs with different value of the rationality parameter $\lambda$. (Right) Expected utility of different algorithms against CLQR (dashed) and BR (solid) in Leduc Hold'em. $p$ is fixed for both regret minimization approaches and QNE is the value of COMB or RQR with $p = 1$. }
\label{leduc_perf}
\vspace{-5pt}
\end{figure}

\subsection{Summary of the results}
In the experiments, we have shown three main points. (1) GA approach does not scale even to moderate games, making regret minimization approaches much better suited to larger games. (2) In both normal-form and extensive-form games, the RQR approach outperforms NASH and QNE baseline in terms of gain and outperforms QNE in terms of exploitability, making it currently the best approach against LQR opponents in large games. (3) Our algorithms perform better than the baselines, even with different rationality values, and can be successfully used even in general games. Visual comparison of the algorithms in zero-sum games is provided in the following table. Scalability denotes how well the algorithm scales to larger games. The marks range from three minuses as the worst to three pluses as the best with NE being the 0 baselines.

\begin{center}
%\resizebox{.99\linewidth}{!}{  %@{~}c@{~}c@{~}c@{~}c@{~}c@{~}c@{~}
\small{
\begin{tabular}{@{~}l|ccccc@{~}}
                    & COMB & \tikzmark{top left 1}RQR  & QNE   &  NE   &  GA \\ \hline
  {\bf Scalability} & -    &  0  &   0      &  0      & -~-~-~\\
  Gain              & ++    &   ++  &    +       &  0     & +++ \\
  Exploitability    & -~-   &  - \tikzmark{bottom right 1}  &     -~-~-    &  0     &  -~-   
\end{tabular}
}
 \DrawBox[ultra thick, rqrcolor]{top left 1}{bottom right 1}
%}
\end{center}
\vspace{-5pt}
\section{Conclusion}
% Bounded rationality models are crucial for applications that involve human decision-makers.
% Most previous results on bounded rationality consider games among humans, where all players' rationality is bounded.
% However, artificial intelligence applications in real-world problems pose a novel challenge of computing the optimal strategies for a perfectly rational system interacting with bounded-rational humans.
% We call this optimal strategy Quantal Stackelberg Equilibrium (QSE) and show that natural adaptations of existing algorithms do not lead to QSE, but rather to a different solution we call Quantal Nash Equilibrium (QNE).
% However, we show that QNE typically achieves lower utility against quantal response opponents compared to QSE. Moreover, QNE has a lower gain than the worst possible Nash equilibrium in some games.
% On the other hand, we show that finding QSE is computationally hard and does not scale to large domains.
% Hence, we propose a variant of counterfactual regret minimization which, based on our experimental evaluation, 
% scales to large games, and computes strategies that outperform QNE in two aspects: it achieves a higher utility against the quantal response opponent and a perfectly rational opponent.

Bounded rationality models are crucial for applications that involve human decision-makers.
Most previous results on bounded rationality consider games among humans, where all players' rationality is bounded. However, artificial intelligence applications in real-world problems pose a novel challenge of computing optimal strategies for an entirely rational system interacting with bounded-rational humans. We call this optimal strategy Quantal Stackelberg Equilibrium (QSE) and show that natural adaptations of existing algorithms do not lead to QSE, but rather to a different solution we call Quantal Nash Equilibrium (QNE). 
As we observe, there is a trade-off between computability and solution quality. QSE provides better strategies, but it is computationally hard and does not scale to large domains. QNE scales significantly better, but it typically achieves lower utility than QSE and might be even worse than the worst Nash equilibrium. Therefore, we propose a variant of counterfactual regret minimization which, based on our experimental evaluation, 
scales to large games, and computes strategies that outperform QNE against both the quantal response opponent and the perfectly rational opponent.

%Bounded rationality in game theory is becoming more important as algorithms are deployed in the real world. In extensive-form games, there are working approaches against static opponents \cite{johanson2008computing,johanson2009data}, but no approaches using dynamic models. We address this by using a dynamic QR model. We propose three new approaches and show that they all perform better than the baseline. One is the GA approach that performs very well on small games but scales poorly to larger games. The other two approaches deal with the scalability issue at the cost of practical performance and also provide a performance guarantee of the NE used in their computation.

%Our work opens a few possible paths for future research. The first is to find some other response function that can be solved optimally using CFR-f. The second one is a continuation of QR research, finding different algorithms with higher exploitation. Last is focusing more on the exploitability part and therefore finding such strategies that are also able to exploit QR but are less exploitable. The first goal could be finding the NE strategy with the best gain.

\section*{Acknowledgement}
This research is supported by Czech Science Foundation (grant no. 18-27483Y) and the SIMTech-NTU Joint Laboratory on Complex Systems. Bo An is partially supported by Singtel Cognitive and Artificial Intelligence Lab for Enterprises (SCALE@NTU), which is a collaboration between Singapore Telecommunications Limited (Singtel) and Nanyang Technological University (NTU) that is funded by the Singapore Government through the Industry Alignment Fund -- Industry Collaboration Projects Grant. Computation resources were provided by the OP VVV MEYS funded project CZ.02.1.01/0.0/0.0/16 019/0000765 Research Center for Informatics.

\bibliography{bibliography}

\cleardoublepage
\appendix

\section{Proofs}

% nastaveni counteru, aby se dobre cislovaly tvrzeni 
\setcounter{theorem}{0}

\subsection{Proof of Theorem~\ref{thm:one}}
\begin{lemma}
Let ${\cal A} = \{a_1,a_2,\dots,a_n\}, a_i\in\real$, $a_1 = \max({\cal A}) > 0$. Then it holds that 
\begin{equation}
    \max({\cal A}) - \softmax_\lambda({\cal A}) \leq \frac{W(1/e)}{\lambda} + \frac{n-2}{\lambda e}.
\end{equation}
\end{lemma}
\begin{proof}
We proceed by induction on the size of the set $\cal A$.

\noindent\underline{Base case:} Let $n=2$. Because $a_2 \leq a_1$, any $a_2$ can be written as $a_1x, x \leq 1$. For a given $\lambda$, the difference between $\max$ and $\softmax$ can be written as
\begin{equation*}
    d(x) = a_1 - \frac{a_1e^{\lambda a_1} + a_1xe^{\lambda a_1x}}{e^{\lambda a_1} + e^{\lambda a_1x}}.
\end{equation*}
To find a maximum of this function, we differentiate it by $x$, which yields
\begin{equation*}
    d'(x) = -\frac{a_1 e^{\lambda a_1x} (e^{\lambda a_1} (\lambda a_1 (x - 1) + 1) + e^{\lambda a_1 x})}{(e^{\lambda a_1 x} + e^{\lambda a_1})^2}.
\end{equation*}
For $a_1>0$, the function $d'$ has a root 
\begin{equation*}
    r = \frac{a_1\lambda - W(1/e) -1}{a_1\lambda},
\end{equation*}
where $W$ is the Lambert function. The root is unique, because the inner function $e^{\lambda a_1} (\lambda a_1 (x - 1) + 1) + e^{\lambda a_1 x}$ is increasing as its derivative is positive for all $x \leq 1$. It is a maximum of $d$, because $d''(r)<0$. By plugging the root into the function $d$, we obtain the upper bound on the distance between $\max$ and $\softmax$:
\begin{equation*}
    d(r) = \frac{W(1/e)}{\lambda},
\end{equation*}
which is independent on $a_1, a_2$.

\noindent\underline{Induction step:} For a given $|{\cal A}|=n$, assume $\max({\cal A}) - \softmax_\lambda({\cal A}) \leq C$. Consider a new $a_{n+1}\leq a_1$. Again, we set $a_{n+1}=a_1x, x\leq 1$. For a given $\lambda$, the difference between $\max$ and $\softmax$ can be written as
\begin{equation*}
    a_1 - \frac{\sum_{i=1}^n a_ie^{\lambda a_i} + a_1xe^{\lambda a_1x}}{\sum_{i=1}^n e^{\lambda a_i} + e^{\lambda a_1x}} \leq C + \frac{(a_1-a_1x)e^{\lambda a_1x}}{e^{\lambda a_1}},
\end{equation*}
because the $exp$ function is strictly greater than zero. To find a maximum of the second term, we differentiate it by $x$:
\begin{equation*}
    \left(\frac{(a_1-a_1x)e^{\lambda a_1x}}{e^{\lambda a_1}}\right)' = a_1^2 \lambda (1 - x) e^{a_1 \lambda (x - 1)} - a_1 e^{a_1 \lambda (x - 1)}.
\end{equation*}
As in the base case, for $a_1>0$ the derivative has a root
\begin{equation*}
    r = 1 - \frac{1}{\lambda a_1}, \qquad \frac{(a_1-a_1r)e^{\lambda a_1r}}{e^{\lambda a_1}} = \frac{1}{\lambda e}.
\end{equation*}
The root is unique, because the derivative is positive increasing on $(-\infty, 1-2/a_1\lambda)$ and decreasing on $(1-2/a_1\lambda,1]$, as differentiating it for the second time reveals. Therefore, we obtain the upper bound
\begin{equation*}
    \max({\cal A}\cup a_{n+1}) - \softmax_\lambda({\cal A}\cup a_{n+1}) \leq C + \frac{1}{\lambda e}.
\end{equation*}
The result follows from the induction. Note that the upper bound goes to zero as $\lambda$ approaches infinity.
\end{proof}

\begin{theorem}
Computing a QNE strategy profile in two-player NFGs is a PPAD-hard problem.
\end{theorem}

\begin{proof}
Let $\tilde{G}$ be a 2-player NFG with strictly positive utilities, in which one of the players has $n$ actions to play. Computing an $\epsilon$-NASH in $\tilde{G}$ is PPAD-complete~\cite{daskalakis2009complexity}. We show that computing QNE is PPAD-hard by reducing the problem of finding $\epsilon$-NASH in $\tilde{G}$ to a problem of computing a specific QNE in $\tilde{G}$.

We construct the reduced game as follows: let the player with $n$ actions be the subrational player and let $q$ be from a logit class, i.e., $q(x) = e^{\lambda x}$ for some $\lambda$. Assume that there exists $\lambda^*$, such that for each $\epsilon$ and each strategy $\sigma_\pr$ of the leader $u_\ps(\sigma_\pr, BR(\sigma_\pr)) - u_\ps(\sigma_\pr, QR(\sigma_\pr)) \leq \epsilon$. Because the leader plays fully rationally, his QNE strategy is a best response. By the definition of $\lambda^*$, the follower's QR is an $\epsilon$-best response. Therefore, by solving for QNE with $q(x) = e^{\lambda^* x}$, we find an $\epsilon$-NASH in $\tilde{G}$.

Each strategy $\sigma$ of a leader generates expected utilities for the follower, playing BR corresponds to $\max$, playing QR corresponds to $\softmax$. Because the game we reduce from has $n$ actions, there are $n$ expected utilities, we can hence use the lemma. Setting $\lambda^* = \frac{W(1/e)}{\epsilon}+\frac{n-2}{\lambda e}$ concludes the proof.
\end{proof}

\subsection{Proof of Proposition~\ref{prop:badqne}}
\begin{proposition}
 For any $LQR$ function. There exists a zero-sum normal-form game $G = (N,A,u)$ with a unique NE~-~$(\sigma_\pr^{NE},\sigma_\ps^{NE})$ and QNE - $(\sigma_\pr^{QNE},QR(\sigma_\pr^{QNE}))$ such that $u_\pr(\sigma_\pr^{NE},QR(\sigma_\pr^{NE})) > u_\pr(\sigma_\pr^{QNE},QR(\sigma_\pr^{QNE}))$.
\end{proposition}
\begin{center}
 \begin{tabular}{cccc}
  & A & B & C \\ \cline{2-4}
 X & \multicolumn{1}{|c|}{-6} & \multicolumn{1}{|c|}{9}  & \multicolumn{1}{|c|}{9} \\ \cline{2-4}
 Y & \multicolumn{1}{|c|}{3} & \multicolumn{1}{|c|}{0} & \multicolumn{1}{|c|}{2} \\ \cline{2-4}
 \end{tabular}
 \end{center}
 \begin{proof}
 In the provided game, the only NE for player $\pr$ is $(\frac{1}{6},\frac{5}{6})$ and QR against it results in expected utility 1.6438. QNE strategy for player $\pr$ is (0.1744,0.8256) resulting in expected utility 1.6366. Therefore, in this game with $\lambda = 1$ QNE is worse than NE against QR. For a different $\lambda>0$, the utilities can be re-scaled by $\frac{1}{\lambda}$ to achieve a similar result.
\end{proof}

\subsection{Proof of Theorem~\ref{thm:np}}
\setcounter{theorem}{3}

\begin{theorem}
Let $G$ be a two-player imperfect-information EFG with perfect recall and $QR$ be a quantal response function. Computing an optimal strategy of a rational player against the quantal response opponent in $G$ is an NP-hard if one of the following holds:
(1) $G$ is zero-sum and $QR$ is generated by a logit generator $q(x) = exp(\lambda x)$ for some $\lambda>0$;
or (2) $G$ is general sum.
\end{theorem}
\begin{proof}
We reduce the problem of solving an instances of the partition problem to finding QSE in a specific zero-sum EFG. An instance of a partition problem is a multiset of positive integers $(x_i)_{i\in\{1\dots n\}}$. The question is whether there is a set of indices $I \subset \{1\dots n\}$, such that 
\begin{equation}
\label{eq:0sqse:part}
    \sum_{i\in I} x_i = \sum_{i \in \{1\dots n\} \setminus I} x_i.
\end{equation}

For constructing the game we use a special NFG with two distinct QSEs, which are different from the uniform strategy. An example of such a NFG is depicted in Figure~\ref{fig:nfg:0sqse}. 

\begin{figure}[h]
\begin{minipage}{0.59\linewidth}
\centering
\includegraphics[width=\linewidth]{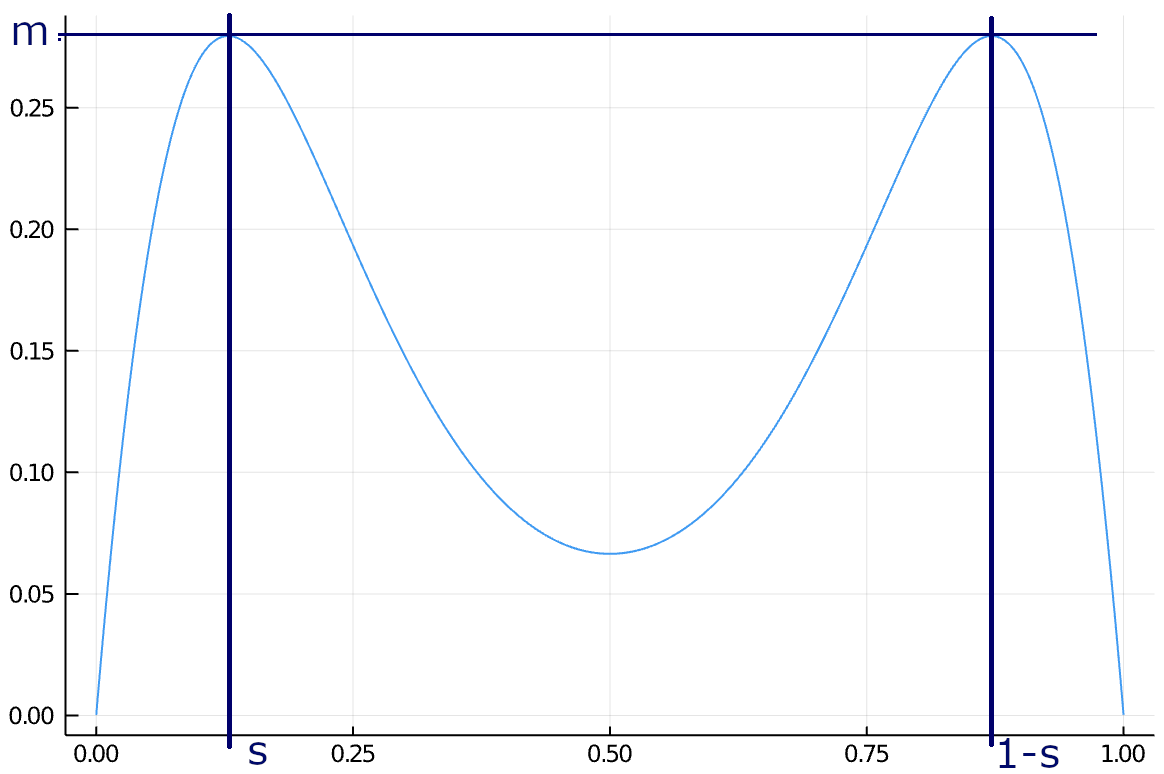}
\end{minipage}
\begin{minipage}{0.39\linewidth}
\scriptsize
 \centering
 \begin{tabular}{cccc}
  & A & B & C \\ \cline{2-4}
 X & \multicolumn{1}{|c|}{0} & \multicolumn{1}{|c|}{10}  & \multicolumn{1}{|c|}{0}  \\ \cline{2-4}
 Y & \multicolumn{1}{|c|}{0} & \multicolumn{1}{|c|}{0} & \multicolumn{1}{|c|}{10}  \\ \cline{2-4}
 \end{tabular}
\end{minipage}
\caption{ An NFG and a criterion function of its QSE with generator $q(x) = exp(x)$.
}
\label{fig:nfg:0sqse}
\end{figure}

In the first equilibrium, the rational player $\pr$ plays the first action with probability $s$. The second equilibrium is when she plays the second action with probability $s$. The expected reward of $\pr$ when playing either of these strategies is $m$, while any other strategy, and particularly the uniform strategy, achieves a lower reward.

Now we can proceed to constructing the game, which makes the rational player to commit to a strategy that solve the partition problem. The game starts with a uniform chance node. For each item, there is a subgame as indicated in the game in Figure~\ref{fig:efg:0sqse} (for two of the items $x_i$ and $x_j$).

\begin{figure}[h]
    \centering
    \includegraphics[width=0.95\linewidth]{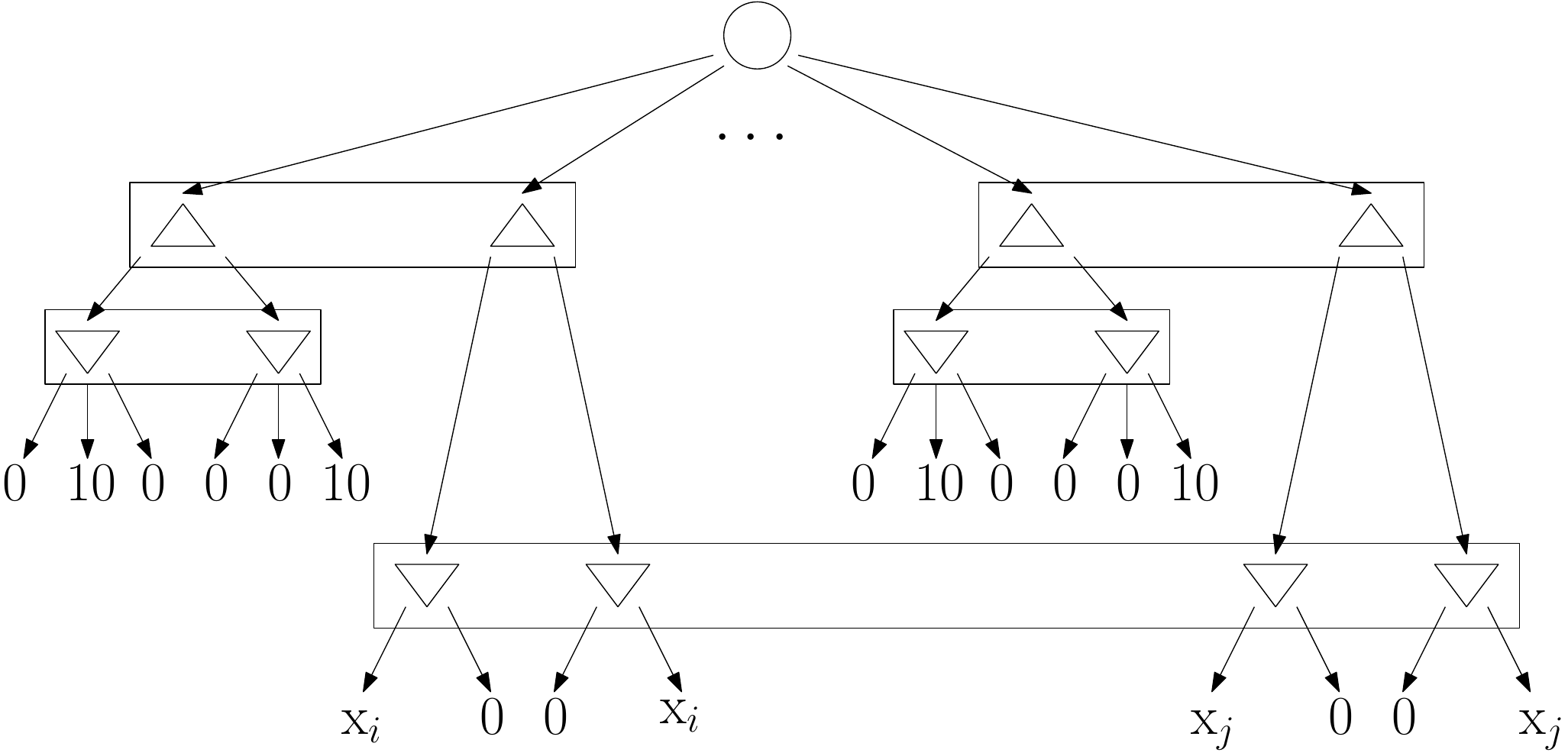}
    \caption{A constructed EFG for a partition problem.}
    \label{fig:efg:0sqse}
\end{figure}

There are two main components of each subgame. The first component (on the left) -- the NFG subtree -- is the EFG representation of the NFG game introduced earlier. To maximize her utility, the rational player $\pr$ is motivated to play either the first action with probability $s$ or $1-s$, but not a uniform strategy. The second part -- the partition subtree -- solves the partition problem.
\subsubsection{Solvable instances}
First, we construct the QSE of this game in case the partition problem has a solution, i.e., there exists an index set $I$ for which Eq.~\eqref{eq:0sqse:part} holds. To maximize the utility in the NFG subtrees, in each of her information sets player $\pr$ chooses only from the two strategies $s$ and $1-s$. For each item, if the item belongs to the set $I$, she chooses the strategy $s$. If she chooses strategy $1-s$, it means the item is from the complementary set. The expected utilities of player $\ps$ of actions $a_1, a_2$ in the lower information set are
\begin{align*}
    u_\ps(a_1) &= \frac{1}{2n}\left(-\sum_{i\in I}sx_i - \sum_{i\in\{1\dots n\} \setminus I}(1-s)x_i\right) \\
    u_\ps(a_2) &= \frac{1}{2n}\left(-\sum_{i\in I}(1-s)x_i - \sum_{i\in\{1\dots n\} \setminus I}sx_i\right).
\end{align*}
Because $I$ is the solution, we have $u_\ps(a_1) = u_\ps(a_2)$ and player $\ps$ is incentivized to play uniformly. The $\pr$'s utility in the partition subtrees is hence
\begin{equation*}
    u^U_\pr = \sum_{i\in\{1,\dots,n\}}\frac{x_i}{4n} = \frac{1}{4n}\sum_{i\in\{1,\dots,n\}}x_i.
\end{equation*}
Next, we show that utility $ u^U_\pr$ is optimal in the partition subtrees -- player $\pr$ can never get a higher utility. Let $x$ be a vector of the multiset integers of the partition problem and $\sigma$ be a vector of arbitrary probabilities of playing the first action in player $\pr$'s partition subtrees. We aim to prove that for any $\sigma$ and the corresponding vector of complementary probabilities of playing the second actions $\overline{1}-\sigma$ it holds that
\begin{equation*}
    \frac{1}{2n}\frac{x^T\sigma q(-x^T\sigma/2n) + x^T(\overline{1}-\sigma)q(-x^T(\overline{1}-\sigma)/2n)}{q(-x^T\sigma/2n)+q(-x^T(\overline{1}-\sigma)/2n)} \leq u^U_\pr.
\end{equation*}
Simple algebra shows this is equivalent to
\begin{equation}
\label{eq:0sqse:ub}
    x^T(\sigma-\overline{1/2})(q(-x^T\sigma/2n)-q(-x^T(\overline{1}-\sigma)/2n)) \leq 0.
\end{equation}
Because we have
\begin{equation*}
    q(-x^T\sigma/2n)-q(-x^T(\overline{1}-\sigma)/2n) \leq 0 \iff x^T(\sigma-\overline{1/2}) \geq 0,
\end{equation*}
Eq.~\eqref{eq:0sqse:ub} always holds and $u^U_\pr$ is indeed an upper bound. Because player $\pr$'s utility is maximized in both the NFG and the partition subtrees, it is a QSE and her utility if the partition problem is solvable is therefore 
\begin{equation*}
    u^*_\pr = m/2 + 1/2n\sum_{i\in\{1,\dots,n\}}x_i.
\end{equation*}

% 1/ If the leader is able to incentivize a uniform strategy of the follower, it is optimal in lower information set.
% - compute leader utility ul for uniform
% - show it is optimal by showing ul is upper bound.

% 2/ If there exists solution of the partition problem then the leader is able to incentivize uniform strategy of the follower in the lower information set.

\subsubsection{Unsolvable instances}

Second, assume that the partition problem does not have a solution. We show that in this case, the utility of player $\pr$ in the QSE will be always strictly lower than $u^*_\pr$. Observe that because QSE with solvable instances achieves a maximum possible utility in the partition subtrees, in order to attempt to reach the same overall utility with unsolvable instances, player $\pr$ has to commit to the solution of the NFG game. Therefore, in each partition subtree, her only viable strategy is to play the first action with probability either $s$ or $1-s$. First, we analyze the utility of player $\pr$ in case the strategy of player $\ps$ is not uniform. From Eq.~\eqref{eq:0sqse:ub}, it follows that in case a vector $\sigma$ maximizes a utility of player $\pr$, it holds that
\begin{equation*}
    x^T(\sigma-\overline{1/2})(q(-x^T\sigma/2n)-q(-x^T(\overline{1}-\sigma)/2n)) = 0.
\end{equation*}
Consequently, if the strategy is not uniform, the difference in quantal functions is nonzero and it is easy to show that also the scalar product $x^T(\sigma-\overline{1/2})$ never reaches zero, thus, making impossible for a non-uniform strategy to be optimal. Therefore, to achieve utility $u^*_\pr$, player $\pr$ has to enforce a uniform strategy of player $\ps$. Given that player $\pr$ has to commit to either $s$ or $1-s$ in her upper information sets, we analyze the conditions when player $\ps$ is incentivized to play a uniform strategy. Let the set $I$ be defined similarly as earlier: an item belongs to $I$ if the first action in player $\pr$'s partition subtree is played with probability $s$. We have
\begin{equation*}
    u_\ps(a_1) = u_\ps(a_2) \iff (1-2s)\smashoperator{\sum_{i\in\{1\dots n\}}}x_i + (2s-1)\smashoperator{\sum_{i\in\{1\dots n\} \setminus I}}x_i = 0.
\end{equation*}
Because there is no $I$ such that the sums are equal and because by the setting of the NFG game $s\neq 1-s$, player $\pr$ never simultaneously enforces optimal utility in the NFG game and the partition subtrees. Her utility is hence strictly smaller than $u^*_\pr$. By analyzing the QSE of the reduced game we hence separate solvable and unsolvable instances of the partition problem.

\subsubsection{General-sum games}

The situation in non-zero-sum games is even simpler. The structure of the proof is exactly as the proof for zero-sum games above, but the role of the NFG subtree can be played by the cooperative coordination game:
\begin{center}
\scriptsize
 \begin{tabular}{ccc}
  & A & B \\ \cline{2-3}
 X & \multicolumn{1}{|c|}{1,1} & \multicolumn{1}{|c|}{0,0} \\ \cline{2-3}
 Y & \multicolumn{1}{|c|}{0,0} & \multicolumn{1}{|c|}{1,1} \\ \cline{2-3}
 \end{tabular}
\end{center}
For any quantal response function, player $\ps$ plays the action with higher expected utility with a higher probability. Therefore, the uniform strategy for player $\pr$ corresponds to the strict minimum of his utility achievable against any quantal opponent. Any other strategy will make the two actions of player $\ps$ have different expected utilities and hence the better will be played with probability more than 0.5, giving player $\pr$ better reward than the uniform strategy.
Since the game is completely symmetric, it has two distinct QSEs.

A similar argument holds also for the partition subtree, which stays unchanged from the zero-sum game. In solvable instances, player $\pr$'s commitment makes any quantal player be indifferent and play uniformly. In case of unsolvable instance, one of her action will be better and played with a strictly higher probability. This will give player $\ps$ more utility than the uniform strategy and hence it would be suboptimal for player $\pr$.

% 3/ argue qse is solution for partition problem

% - not binary in upper is -> maybe optimal (uniform) lower (eg uniform upper), but suboptimal in NFG

% - binary in upper and not solution -> not uniform in lower a1<a2 .. q(-a1) > q(a2) ... leaders utility decreased ... suboptimal 

% The second part solves the partition problem, assuming player $\pr$ chooses only from the two strategies above. For each item, if she chooses the strategy $s$, she is assigning the item to the first set. If she chooses strategy $1-s$, it means the assignment to the second set. The two actions in the lower information set of player $\ps$ represents the two sets in the partition problem. If player $\pr$ successfully solves the partition problem, both the actions will have the same expected value, player $\ps$ will play uniformly and the expected value of the information set will be $\frac{1}{2} \sum_{i\in\{1\dots n\}} x_i$. If player $\pr$ does not successfully solve the partition problem, one of the actions will have a lower utility for player $\pr$. Even the subrational player will play this action with a higher probability and hence lower the overall reward.

% The instance of the partition problem is positive if and only if the value of QSE for player $\pr$ is equal to $\frac{1}{2}m+\frac{1}{2} \sum_{i\in\{1\dots n\}} x_i$.
\end{proof}

\setcounter{theorem}{5}

\subsection{Proof of Proposition~\ref{prop:RMQRconv}}
\begin{proposition}
Let $G = (N,A,u)$ be a zero-sum NFG, $QR$ a quantal response function of the follower, which depends only on the ordering of expected utilities of individual actions. Then the RM-QR algorithm converges to QSE.
\end{proposition}
\begin{proof}
A response function $f$ is called a pretty-good-response if it satisfies 
\begin{equation}
\label{eq:pgr}
    u_\ps(\sigma_\pr, f(\sigma_\pr)) \geq u_\ps(\sigma_\pr, f(\sigma_\pr')) \quad\forall \sigma_\pr, \sigma_\pr'\in\Sigma_\pr.
\end{equation}
Let $QR$ be a simple quantal response function of the follower, which depends only on the descending ordering of expected utilities of follower's actions and consider two different $\sigma_\pr, \sigma_\pr'\in\sigma_\pr$. In case $\sigma_\pr$ induces the same ordering as  $\sigma_\pr'$, then $u_\ps(\sigma_\pr, QR(\sigma_\pr)) = u_\ps(\sigma_\pr, QR(\sigma_\pr'))$. Let $\sigma_\pr$ induce an ordering of indices $i_1, i_2, \dots, i_n$ and $\sigma_\pr'$ induce a different ordering $j_1, j_2, \dots, j_n$. By definition of a quantal function, $QR(\sigma_\pr,a^{i_1})\geq QR(\sigma_\pr,a^{i_2})\geq\dots \geq QR(\sigma_\pr,a^{i_n})$ and $QR(\sigma_\pr',^{j_1})\geq QR(\sigma_\pr',a^{j_2})\geq\dots \geq QR(\sigma_\pr',^{j_n})$. For each $k\in [n]$ it holds that $u_\ps(\sigma_\pr, a^{i_k})QR(\sigma_\pr,a^{i_k}) \geq u_\ps(\sigma_\pr, a^{i_k})QR(\sigma_\pr',a^{j_k})$ and therefore $u_\ps(\sigma_\pr, QR(\sigma_\pr)) \geq u_\ps(\sigma_\pr, QR(\sigma_\pr'))$. Simple QR is hence a pretty-good-response and RMQR converges to a strategy exploiting pretty-good-responses the most, which is a QSE strategy.
\end{proof}

\subsection{Proof of Proposition~\ref{prop:notPGR}}
\begin{proposition}
Let $QR$ be canonical quantal function with a strictly monotonically increasing generator $q$. Then $QR$ is not a pretty-good-response.
\end{proposition}
\begin{figure}[h!]
\centering
 \begin{tabular}{ccc}
  & \multicolumn{2}{c}{Game 3} \\
  & A & B \\ \cline{2-3}
 X & \multicolumn{1}{|c|}{b} & \multicolumn{1}{|c|}{a}  \\ \cline{2-3}
 Y & \multicolumn{1}{|c|}{c} & \multicolumn{1}{|c|}{a}  \\ \cline{2-3}
 \end{tabular}
 \caption{An example of NFG for which no monotonically increasing canonical quantal function constitutes a pretty-good-response.}
 \label{fig:npgr_ap}
 \end{figure}
\begin{proof}
In Figure~\ref{fig:npgr_ap}, we construct a game $G$ with $A_\ps = (A,B)$, such that no canonical quantal function is a pretty-good-response in this game. Let $a,b,c\in\real$, such that $a<b<c$. Since $q$ is strictly monotonically increasing, we have $q(a)<q(b)<q(c)$. By the definition of canonical quantal response, we have $QR(Y,A) - QR(X,A) = QR(Y,B) - QR(X,B)$. Because $q(b)<q(c)$, both sides of the equation are positive. Since $a<b$ it holds that $b(QR(Y,A) - QR(X,A)) > a(QR(Y,B) - QR(X,B))$, therefore $bQR(Y,A) + aQR(Y,B) > bQR(X,A) + aQR(X,B)$ and finally $u_\ps(X, QR(Y)) > u_\ps(X, QR(X))$. By definition in Equation~(\ref{eq:pgr}), $QR$ is hence not a pretty-good-response.
\end{proof}

\section{Evaluation}
\paragraph{Experimental setup.}  For all experiments except Goofspiel 7, we use Python 3.7. We solve non-linear optimization using the SLSQP GA from the SciPy 1.3.1 library. LP computations are done using gurobi 8.1.1, and experiments were done on Intel i7 1.8GHz CPU with 8GB RAM. Goofspiel experiment was run on 24 cores/48 threads 3.2GHz (2 x Intel Xeon Scalable Gold 6146) with 384GB of RAM, implemented in C++. For experiments on zero-sum NFGs, we used randomly generated square games and for general-sum NFGs we used randomly generated games, Grab the Dollar, Majority Voting, Traveler's Dilema and War of Attrition from GAMUT~\cite{nudelman2004run}. For EFGs, we used randomly generated sequential games, and Leduc Hold'em. In the experiments, we wanted to measure the scalability and performance of the proposed solutions and the baseline.

\paragraph{Domains.} \textbf{\textit{Randomly Generated NFGs}} are parametrized by sizes of both players' action spaces. Utilities are generated uniformly at random from integers between -9 and 10. \textbf{\textit{Grab the Dollar}} is a game with a prize that both players can grab at any given time, actions being the times. If both players grab it at the same time they both receive low payoff and when one player grabs the price before the opponent she receives high payoff and the opponent payoff somewhere between high and low. In \textbf{\textit{Majority Voting}} the players have utilities assigned to each action (candidate) being declared winner. And the winner is the candidate with the most votes. In a tie a candidate with higher priority is declared winner. \textbf{\textit{Travelers Dillema}} is a game where both players propose a payoff and the player with lower proposal wins the payoff plus some bonus and the opponent receives the payoff minus some bonus. In a \textbf{\textit{War of Attrition}}, two players are in a dispute over an object, and each chooses a time to concede the object to the other player. If both concede at the same time, they share the object. Each player has a valuation of the object, and each player’s utility is decremented at every time step. \textbf{\textit{Randomly Generated EFGs}} are EFGs where players switch each turn. The game has three parameters. One is the branching factor $b$, the second is the maximal number of observations received $o$, and the last one is maximal sequence length for one player $l$. Therefore, the maximal depth is $2l$. The path from the root correlates utilities, and the generation of utilities proceeds as follows. The value is set to 0 at the root and randomly changes by one up or down each time when moving to the children. The utility of a history is the value with which the leaf node is reached. We generated four sets in following way. Set \textbf{1:} $b=3,o=2,l=1$, \textbf{2:} $b=3,o=2,l=2$, \textbf{3:} $b=5,o=3,l=2$, \textbf{4:} $b=5, o=3,l=3$. During the generation we discarded the games where NE strategy was the same as GA strategy because such degenerate games would have all the values that we report the same. We kept generating and discarding until we had 300 games in each set. Number of games we had to generate in order to obtain 300 non degenerate games in each set is: \textbf{1} - 3571, \textbf{2} - 607, \textbf{3} - 479, \textbf{4} - 317.For \textbf{\textit{Leduc Hold'em}} we use the definition from~\cite{lockhart2019computing}. \textbf{\textit{Goofspiel 7}} is a bidding card game where players are trying to obtain the most points. Cards are shuffled and set face-down. Each turn, the top point card is revealed, and players simultaneously play a bid card; the point card is given to the highest bidder or discarded if the bids are equal. In this implementation, we use a fixed deck with K = 7.

\newpage
\section{Mathematical program to solve QSE}
\begin{observation}
\label{obs:qse_efg_mp}
Let $G$ be an extensive-form game and a $q$ be a generator of a canonical quantal function. Then QSE of $G$ can be formulated as a following non-concave mathematical program:
\begin{flalign}
\max_{r_\pr} &~v_\pr(root) \label{eq:qse_efg_max}\\
 r_i(\emptyset) &= 1 \quad \forall i \in N \label{eq:qse_efg_plan_0}\\
0 &\leq r_i(s_i) \leq 1 \quad  \forall s_i \in S_i, \forall i \in N \label{eq:qse_efg_plans_prob} \\
r_i(s_i) &= \sum_{a\in A_i(I_i)}r_i(s_ia) \label{eq:qse_efg_sum_of_plans}\\
&\forall s_i \in S_i, I_i \in inf_i(s_i), \forall i \in N \nonumber \\
v_i(I) &= \sum_{a \in A_i(I)}f_i(I,a)r_i(s_ia) \label{eq:qse_efg_endvalue}\\
&\forall I \in \mathcal{I}_i, s_i = seq_i(I), \forall i \in N \nonumber \\
r_\ps(s_\ps a) &= \frac{r_\ps(s_\ps)q({f_\ps(I,a)})}{\sum_{a \in A_\ps(I)q({f_\ps(I,a)})}} \label{eq:qse_efg_nonl2}\\ 
&\forall s_\ps \in S_\ps, I = inf_\ps(s_\ps), \forall a \in A_\ps(I)  \nonumber \\
f_\ps(I,a) &= \nonumber \\
&\sum_{I' \in \mathcal{I}_\ps:s_\ps a = seq_\ps(I')}v(I') + \sum_{s_\pr \in S_\pr}u(s_\pr,s_\ps a)r_\pr(s_\pr) \label{eq:qse_efg_f_one}\;\;\; \\ 
&\forall I \in \mathcal{I}_\ps, s_\ps = seq_\ps(I), \forall a \in A_\ps(I) \nonumber \\
f_\pr(I,a) &= \nonumber \\
&\sum_{I' \in \mathcal{I}_\pr:s_\pr a = seq_\pr(I')}v(I') + \sum_{s_\ps \in S_\ps}u(s_\pr a,s_\ps)r_\ps(s_\ps) \label{eq:qse_efg_f_two}\;\;\; \\ 
&\forall I \in \mathcal{I}_\pr, s_\pr = seq_\pr(I), \forall a \in A_\pr(I) \nonumber 
\end{flalign}
\end{observation}

Equation~\ref{eq:qse_efg_max} is for maximizing the expected value of player $\pr$ in the root over his realization plans. 

Equation~\ref{eq:qse_efg_plan_0} fixes probability of empty realization plan to 1, Equation~\ref{eq:qse_efg_plans_prob} constraints realization plans as probabilities and Equation~\ref{eq:qse_efg_sum_of_plans} defines the relationship of child plans to their parents. Equation~\ref{eq:qse_efg_endvalue} defines $v_i(I)$ as sum of values in each children times the realization plan there. Equation~\ref{eq:qse_efg_nonl2} defines the quantal response in realization plans of player $\ps$. Finally Equations~\ref{eq:qse_efg_f_one} and \ref{eq:qse_efg_f_two} define the action value summing over both descendant infosets and terminal nodes. Because of Equation~(\ref{eq:qse_efg_nonl2}), the program is not linear. The problem of computing the QSE is computationally difficult to solve -- it is an NP-hard problem.

\newpage
\section{Scalability on NFGs.}
Figure~\ref{nfg_scal} shows running time averaged over 1000 games for each size of square zero-sum NFGs. Ranging from 2 actions up to 377 actions.
\begin{figure}[h!]
\includegraphics[width=\linewidth]{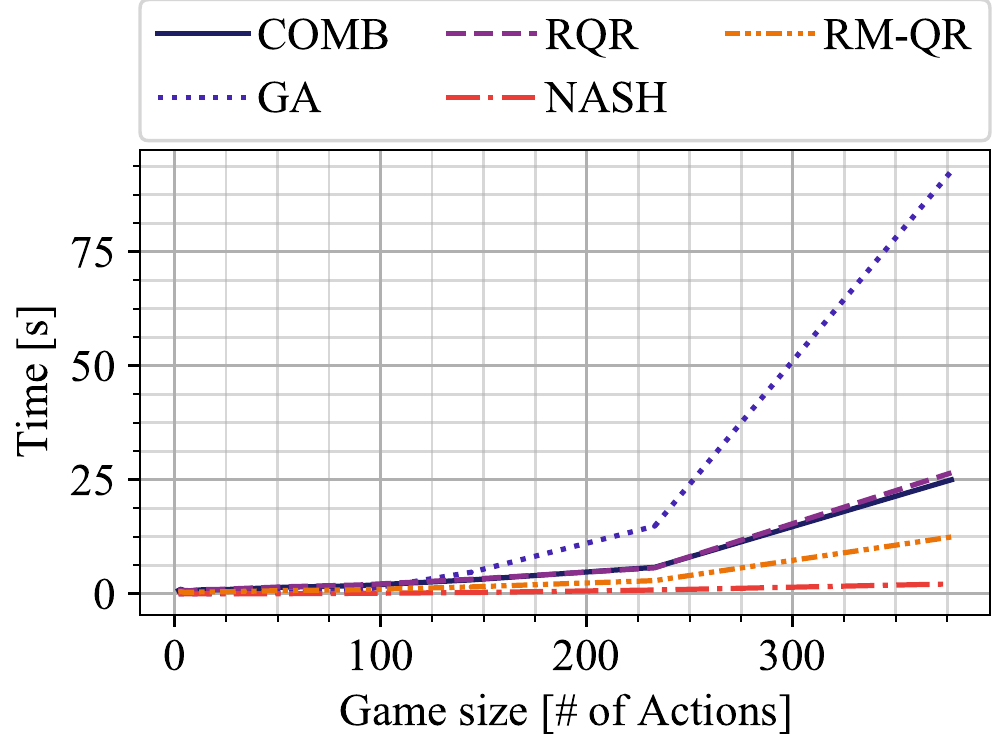}
\caption{Running time comparison of COMB, GA, RQR, NASH(SE), and QNE on zero-sum NFGs}
\label{nfg_scal}
\end{figure}

\section{One card poker results.}
Figure~\ref{ocp} shows the expected utility of the COMB and RQR when run with fixed $p$, for different values of $p$. CFR results are on left end of the RQR and COMB lines and CFR-QR is on the right end.
\begin{figure}[h!]
\includegraphics[width=\linewidth]{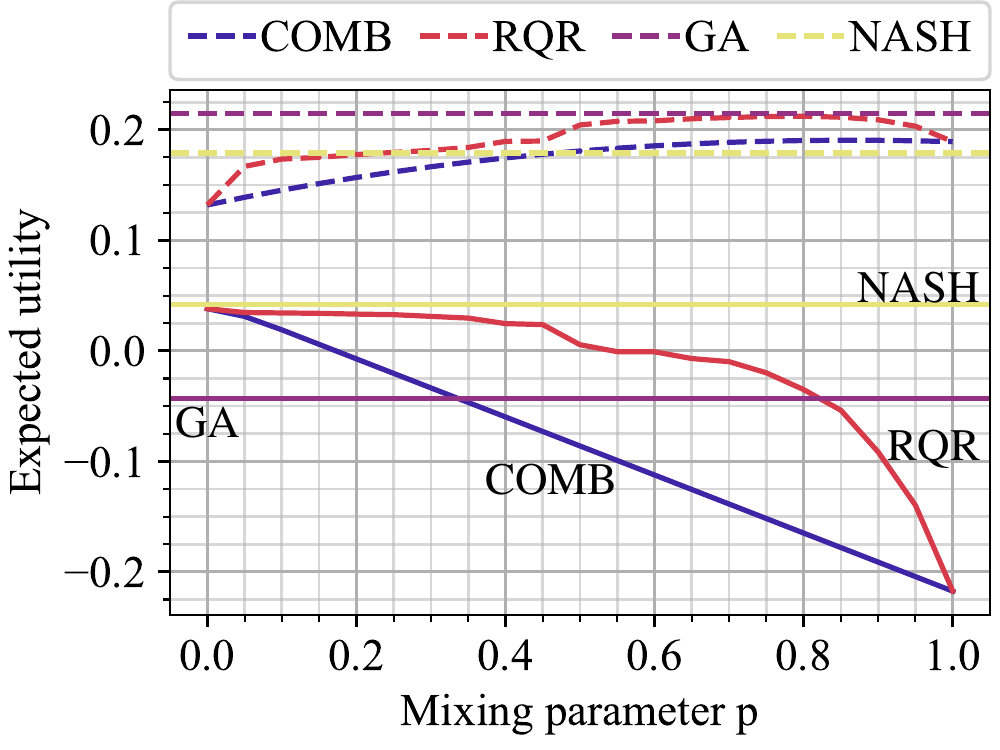}
\caption{Expected utility for different algorithms against CLQR (dashed) and BR (solid) in one card poker. $p$ is a constant for both regret minimization approaches. NASH and GA are also reported and CFR-QR is the value of COMB or RQR with $p = 1$.}
\label{ocp}
\end{figure}

\section{Attempts to Prove Conjecture~\ref{convergence}}
\textbf{CFR-f} soundness proof \cite{davis2014using} %proves optimal strategy against QR which are also pretty-good responses
relies on the definition of pretty-good-responses (which do not include a majority of quantal responses) to prove optimality of the resulting CFR strategy. We tried to adapt it in a way that we exchange the $\max_{\sigma^*_1 \in \Sigma_1}u_1(\sigma^*_1,f(\sigma^*_1))$ for QNE, which is 
\begin{equation*}
\min_{\sigma^*_1 \in \Sigma_1}\max_{\sigma'_1 \in \Sigma_1}u_1(\sigma'_1, f(\sigma^*_1))-u_1(\sigma^*_1,f(\sigma^*_1)).
\end{equation*}
We aim to show that the CFR-QR's average strategy $\bar{\sigma}$ is close to the best response: 
\begin{equation*}
    \max_{\sigma'_1 \in \Sigma_1}u_1(\sigma'_1, f(\bar{\sigma}_1))-u_1(\bar{\sigma}_1,f(\bar{\sigma}_1)) \leq\epsilon.
\end{equation*}

However, the key step of the original proof requires exchanging the opponent's strategy using the pretty-good-response property, and without the property, the required inequality does not hold.

The proof of \textbf{CFR-BR} \cite{johanson2012finding} uses the link to folk theorem and the fact that the responding player has no positive regret, which obviously does not hold in case of bounded rationality when the regret is always non-zero.

The \textbf{current-strategy convergence} of CFR-BR in \cite{lockhart2019computing} shows that we can use the regret of CFR-BR to bound a regret in form of

\begin{equation*}
    R^T = \sum_{t=1}^T\mathit{l}(\sigma_1^t)-\inf_{\sigma_1^* \in \Sigma_1}\sum_{t=1}^T\mathit{l}(\sigma_1^*).
\end{equation*}

This approach can be used with loss function defined as $\mathit{l}(\sigma_1) = -u_1(\sigma_1, f(\sigma_1))$ to show that CFR-f with a pretty-good-response converges in current iterations. However, when used for approximating a saddle point with loss $\mathit{l}(\sigma_1) = \max_{\sigma^*_1 \in \Sigma_1}u_1(\sigma^*_1, f(\sigma_1))-u_1(\sigma_1, f(\sigma_1))$, we are no longer able to bound the changed regret using the original regret from CFR-f.

\section{Dynamics for Optimizing \texorpdfstring{$\mathbf{p}$}{p} in RQR}
An important part of RQR is the optimization of parameter $p$. We use a simple adaptive algorithm with 3 hyperparameters, initialized as $step = 0.01$, $decay = 2^{-1/iterations}$ and $threshold = 1.00001$. We start with $p=0.5$ and after each update of the rational player's strategy, we determine the direction for moving $p$ by a change in gain. If $new gain$ is more than $old gain \times threshold$ or smaller than $\frac{old gain}{threshold}$ (assuming the gains are positive, otherwise the multiplication and division are switched) we change the value of $p$ by step. In case the previous response was QR and the gain increased, we move $p$ towards QR, otherwise we move it towards BR. Vice versa, if the response was BR and the gain increased, we move $p$ towards BR, otherwise we move it towards QR. At the end of each iteration, we update the $step$ by multiplying it by $decay$.

\section{Hyperparameter Setting}
The values of the hyperparameters for RQR were set based on the initial exploration of the hyperparameter space on different classes of games. We ran RQR with 101 values of $p$ uniformly distributed across 0 and 1 and observed how far we were from the optimum. The selected values achieved an empirical error of less than $0.05$ during 1000, 2000, and 10,000 iterations.

In the COMB algorithm we used a fixed parameter $sweep$ of 11 uniformly distributed values between 0 and 1. Optimization in form of binary search could potentially achieve worse results since the objective is not always convex/concave.

\section{Final Values of \texorpdfstring{$\mathbf{p}$}{p} and \texorpdfstring{$\pmb{\alpha}$}{a} in RQR and COMB}
We analyzed the distributions of final values of parameters $p$ and $\alpha$ in both COMB and RQR. When using COMB, the most interesting pattern we observed is that playing QNE tends to be optimal against highly irrational opponents. In these cases, $\alpha=1$ was usually the final value. As the rationality of the opponent increases, the value of optimal $\alpha$ decreases. However, the exact value is game-specific since the rationality change is equivalent to rescaling utilities. In the reported experiments with rationality $\lambda=1$, $80\%$ of best $\alpha$ values were rather uniformly distributed in $\langle0.5, 1\rangle$. Setting $\alpha=0.8$ led to the best performance most often -- in 18.6\% of the cases.

Similarly to COMB, most $p$ values in RQR were uniformly distributed in $\langle0.5, 1\rangle$. However, we do not know if $p$ converged to the optimal value in main experiments since checking it would require running the second part of RQR many times, significantly increasing the experiments' computational time. Therefore, we analyzed only the data are from the initial experiments intended to tune the hyperparameters of RQR.

\section{Training Against Stronger Opponents}
We thank one of the reviewers for suggesting that training against stronger opponents could potentially improve performance against weaker opponents \cite{davis2014using}. We performed experiments evaluating this hypothesis. Indeed, the results confirmed that training against stronger opponents often results in better strategies against weaker opponents. We tested stronger opponents with new lambda ranging from 1.1-times to 3-times the original lambda. According to the results, an optimal lambda to train against was estimated to be 2-times the original lambda. This new approach yielded stronger strategies than QNE but was still inferior to RQR on average. In a few games, it outperformed even RQR. In those games, training RQR with stronger opponents resulted in even more superior strategies.

\end{document}